\documentclass[lettersize,journal]{IEEEtran}
\pdfoutput=1

\newtheorem{assume}{\bf Assumption}
\newtheorem{thm}{\bf Theorem}

\newtheorem{prop}{\bf Proposition}
\newenvironment{proof}{{\it Proof:}}{\hfill $\blacksquare$\par}

\usepackage{cite}

\usepackage{amsmath,amssymb,amsfonts}
\usepackage{algorithm}
\usepackage{algpseudocode}

\usepackage{graphicx}
\usepackage{textcomp}
\usepackage{xcolor}
\usepackage{soul}
\usepackage{multirow}
\usepackage{subfigure}
\usepackage[utf8]{inputenc}
\usepackage{bm}
\usepackage{diagbox}

\usepackage{amsmath,amsfonts}
\usepackage{array}
\usepackage[caption=false,font=normalsize,labelfont=sf,textfont=sf]{subfig}
\usepackage{stfloats}
\usepackage{url}
\usepackage{verbatim}
\usepackage{graphicx}
\usepackage{cite}
\begin{document}

\title{IMFL-AIGC: Incentive Mechanism Design for Federated Learning Empowered by Artificial Intelligence Generated Content\\
}

%\author{
%	\IEEEauthorblockN{Guangjing Huang, Qiong Wu, Jingyi Li, Xu Chen,}
%	\IEEEauthorblockA{School of Computer Science and Engineering, Sun Yat-sen University, China}
%	\IEEEauthorblockA{\{huanggj27, wuqiong23, lijy573\}@mail2.sysu.edu.cn, chenxu35@mail.sysu.edu.cn}
%}

\author{Guangjing Huang, Qiong Wu, Jingyi Li, Xu Chen%<-this % stops a space
	\IEEEcompsocitemizethanks{
	\IEEEcompsocthanksitem G. Huang, J. Li and X. Chen are with the School of Computer Science and Engineering, Sun Yat-sen University, Guangzhou, China (e-mail: huanggj27@mail2.sysu.edu.cn;  lijy573@mail2.sysu.edu.cn; chenxu35@mail.sysu.edu.cn).
		%\protect\\
		% <-this % stops an unwanted space
    \IEEEcompsocthanksitem Q. Wu is with the Department of Electronic and Computer Engineering, The Hong Kong University of Science and Technology, Hong Kong (e-mail: eeqiongwu@ust.hk).
	}
}

\maketitle

\begin{abstract}
	Federated learning (FL) has emerged as a promising paradigm that enables clients to collaboratively train a shared global model without uploading their local data. To alleviate the heterogeneous data quality among clients, artificial intelligence-generated content (AIGC) can be leveraged as a novel data synthesis technique for FL model performance enhancement. Due to various costs incurred by AIGC-empowered FL (e.g., costs of local model computation and data synthesis), however, clients are usually reluctant to participate in FL without adequate economic incentives, which leads to an unexplored critical issue for enabling AIGC-empowered FL. To fill this gap, we first devise a data quality assessment method for data samples generated by AIGC and rigorously analyze the convergence performance of FL model trained using a blend of authentic and AI-generated data samples. We then propose a data quality-aware incentive mechanism to encourage clients' participation. In light of information asymmetry incurred by clients' private multi-dimensional attributes, we investigate clients' behavior patterns and derive the server's optimal incentive strategies to minimize server's cost in terms of both model accuracy loss and incentive payments for both complete and incomplete information scenarios. Numerical results demonstrate that our proposed mechanism exhibits highest training accuracy and reduces up to $53.34\%$ of the server's cost with real-world datasets, compared with existing benchmark mechanisms.
\end{abstract}


\section{Introduction}
\IEEEPARstart{T}{he} fast proliferation of edge devices (e.g., mobile devices and wearable devices) in modern society has led to the rapid growth of data generated from massive distributed sources, which further promotes the advancement of a wide range of artificial intelligent applications (e.g., autonomous driving and healthcare) \cite{nguyen20216g, zhou2019edge}. However, due to the increasing privacy concerns \cite{voigt2017eu} and limited network bandwidth, the predominant mechanism that gathers extensive data from dispersed devices to the cloud for centralized model training becomes impractical. To reap the benefits of the scattered data without privacy risk, federated learning (FL) \cite{mcmahan2017communication} has emerged as a promising paradigm that enables to collaboratively train a shared global model by aggregating locally-computed updates uploaded by clients (e.g., mobile devices). By decoupling model training from the need of direct access of the local data on the devices, FL realizes distributed and privacy-preserving model training.

\begin{figure}[t]
	\centering
	\includegraphics[width=0.9\linewidth]{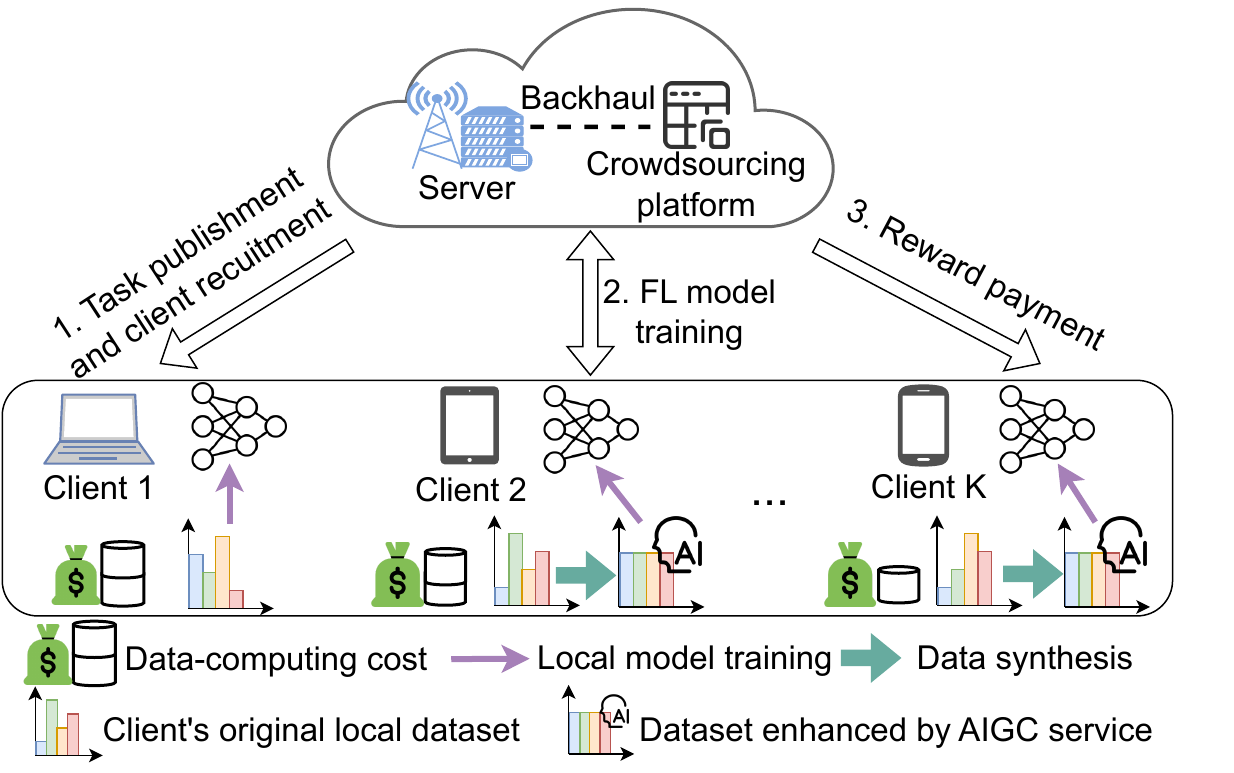}
	\vspace{-10pt}
	\caption{AIGC-empowered federated learning scenario.}
	\vspace{-15pt}
	\label{intro_FL}
\end{figure}

Nevertheless, most existing FL frameworks usually make an optimistic assumption that clients participate in FL model training voluntarily and ignore the inevitable data-computing cost (e.g., battery and CPU resource consumption for local model updates) incurred by inevitable training \cite{shayan2018biscotti}, \cite{yu2020fairness}. While in reality, rational clients would be reluctant to participate in model training without sufficient economic compensation \cite{21incentivetmc}, \cite{24incentive}.
By paying rewards to compensate the cost of clients reasonably, incentive mechanism has garnered significant attention from researchers and become the essential financial catalyst for making FL a reality \cite{pang2022incentive}.

Despite the promising benefits of privacy preservation, federated learning performance remains constrained in mobile application scenarios due to challenges such as non-IID (a.k.a. non independent and identically distributed) nature of data from scattered mobile devices \cite{li2022federated}. 
In certain mobile applications, such as autonomous vehicle training \cite{20vehicles}, \cite{24Synve} and health monitoring \cite{wang2023applications}, the lack of specific labeled data further exacerbates this issue, as the characteristics of users' datasets vary according to their activities.
To alleviate the issue of non-IID data distribution and data scarcity inherent in FL, the great explosion of artificial intelligence-generated content (AIGC) service \cite{du2023enabling}  has opened up a compelling avenue for the clients to generate high-quality data (e.g., images and videos) with generative AI techniques at a rapid pace.
For example, by leveraging generative adversarial networks, clients are able to collectively train a generative model to augment their local data towards yielding an IID dataset \cite{jeong2018communication}. Empowered with generative AI services (e.g., Stable Diffusion \cite{ho2022classifier}), the learning performance of FL can be significantly improved by utilizing a pre-trained diffusion model to synthesize customized local data \cite{li2023filling}. Consequently, integrating AIGC as a data synthesis tool at the client side can substantially improve data quality and enhance FL model performance in real-world applications.

However, the additional costs of data generation by AIGC service would mitigate clients' willingnesses to participate in federated learning without sufficient economic compensation, thereby presenting new challenges to current incentive mechanisms which does not consider the complex clients' behaviors on whether to adopt AIGC service \cite{li2022federated}, \cite{li2023filling}, \cite{24Fedmatch}.
In light of potential benefits and challenges of using AIGC technique in the context of FL, it is non-trivial to devise efficient incentive mechanism for such a complicated AIGC-empowered FL scenario to encourage clients to contribute their data (e.g., local data or rectified data by AIGC service) and resources for FL model training while minimizing server's own cost (e.g., payments for client recruiting, model accuracy loss, etc.). First, an indispensable preliminary step for effective incentive mechanism involves performance assessment of the final converged model prior to commencing FL model training. As the model performance is jointly influenced by many factors, i.e., the number of global training iterations, clients' attributes including the quantity and quality of the local data, it is non-trivial to evaluate the final model performance \cite{nie2018stackelberg, edirimannage2023qarma}. Moreover, the data quality changes resulted from clients' adoption of AIGC service for data synthesis presents key challenges for model performance evaluation. Second, faced the temptation of being rewarded, clients may choose to generate a high-quality dataset by leveraging AIGC service (e.g., reaching IID dataset by replenishing data samples in minority classes) for local model updates, which may further affect the server's decision-making regarding client recruitment, and hence make the design of optimal incentive strategy much more involved. Third, clients' individual  attributes (e.g., data quality and data-computing cost) also bring difficulty to the incentive mechanism design due to information asymmetry (i.e., such private client's information may not be available to the server for decision making prior to the FL training process).

To overcome the aforementioned challenges, in this paper, we propose a data quality-aware incentive mechanism in AIGC-empowered FL scenario. As illustrated in Fig. \ref{intro_FL}, the server publishes a training task and encourages clients' participation through a crowdsourcing platform (e.g., Amazon Mechanical Turk) with a data quality-aware reward allocation mechanism. By analyzing the convergence performance bound in AIGC-empowered FL scenario, we formulate the server's cost in terms of model accuracy loss and payment to clients. Furthermore, we reveal server's optimal incentive strategy to minimize its cost by fully studying clients' rational behaviors under different information settings (i.e., complete and incomplete information scenarios).

In summary, this paper makes the following contributions:

\begin{itemize}
	\item \textit{Incentive mechanism design for AIGC-empowered FL}. We first devise a data quality assessment method for data samples generated by AIGC and then propose a data quality-aware incentive mechanism to encourage clients' participation. By characterizing the clients' complex behavior patterns (e.g., whether to participate in FL, whether to use generated data samples), we design a data quality-aware reward allocation mechanism for the participating clients and derive the server's optimal incentive strategy. To the best of our knowledge, this is the first paper to study the incentive mechanism design for AIGC-empowered FL.
	%	on studying the incentive mechanism for crowdsourcing FL empowered with generated AI service. We provide the method to assess data quality for datasets augmented by AIGC service, which can be easily applicable in the existing  convergence analysis framework. 
	\item \textit{Analysis of model training performance for AIGC-empowered FL}. We derive a novel convergence upper bound of the gap between FL model training loss and the optimal loss value in AIGC-empowered FL scenarios where clients may choose to adopt AIGC-enhanced dataset for local model training instead of their original local dataset.	
	%when introducing generative AI service. %Based on the theoretical analysis results, we formulate the server's cost by striking a trade-off between model accuracy loss and payments for the participating clients, which is to be minimized for server's optimal incentive strategy. %With the goal of server's cost minimization, we investigate the clients' behaviors and derive the server's optimal incentive strategies.
	%This performance bound characterizes impact of client's attribute on the model accuracy loss.
	\item \textit{Investigation on the impact of information asymmetry and the adoption of AIGC service}.
	Compared with complete information scenario, the server may suffer from high cost to hedge the risk of no or very few clients participating in FL model training in incomplete information scenario. Hence, the server tends to boost the incentive reward especially in the case with a small number of candidate clients. Facing with small candidate client size of lower data quality, the adoption of AIGC service for data sysnthesis can bring a much more significant gain in server's cost reduction for both complete and incomplete information scenarios.
	
	%	Introducing AIGC service in FL causes a significant server's cost reduction in the group with low number of candidate client with lower data quality both in complete and incomplete information scenario.
	
	%\textcolor{red}{text}Investigation The optimal incentive strategy for the server under complete and incomplete information scenarios}: We characterize the rational clients' behaviors under the comple	
	%We characterize the clients' behaviors under the complete information scenario where the server knows the attributes of each client, thereby derive the server's optimal incentive strategy. In incomplete information scenario, we summarize distribution of clients' behaviors, based on which, we propose the method to estimate the excepted server's cost. Further, we derive the optimal incentive strategy and the setting of global FL iterations under incomplete information scenario.
	\item \textit{Performance evaluation}. Extensive performance evaluations based on real-world datasets validate our theoretical analysis and show that our incentive mechanism exhibits superior performance, e.g., achieving up-to 53.34\% server's cost reduction and highest training accuracy, compared with the benchmark mechanisms. 
	%. Compared with two benchmarks, our incentive mechanism exhibits the lowest server's cost and the highest training accuracy. 	
	%Numerical results validate our theoretical analysis. Compared with two benchmarks, our mechanism exhibits the lowest server' cost and the highest training accuracy based on real-world datasets.
	
\end{itemize}

\section{System Model}
In this section, we first introduce AIGC-empowered FL paradigm. We further characterize clients' attributes, clients' utility functions and server's utility function. Finally, we propose data quality-aware incentive mechanism under complete and incomplete information scenarios. For readability, we summarize the key notations in Table \ref{table_all_ke}.

\subsection{AIGC-empowered Federated Learning}\label{ger_a}
As illustrated in Fig. \ref{intro_FL}, we consider a AIGC-empowered FL scenario, which comprises of one server and a set $\mathcal{K}=\{1,...,K\}$ of candidate clients who can choose to leverage AIGC service for high-quality data synthesis\footnote{As an initial thrust and for ease of exposition, in this paper we focus on the scenarios that the client candidates who have their local datasets and meanwhile the required capability for data synthesis  (e.g., subscribers of some AIGC services) will be considered for FL, due to the wide penetration of AIGC adoptions worldwide. We will further study the more general cases wherein some clients do not possess the capability for data synthesis.}. As a result, each client can participate in FL with its original local dataset  $\mathcal{D}_{k}$ or AIGC-enhanced dataset $\mathcal{D}_{k}^{A}$ aided with data synthesis.

\begin{table}[t]
	\setlength{\abovecaptionskip}{-0.02cm}
	\renewcommand{\arraystretch}{1.3}
	\caption{List of key notations.}
	\label{table_all_ke}
	\centering
	\scriptsize
	\begin{tabular}{c|c}
		\hline
		Symbol & Description \\
		\hline
		$\mathcal{K}$ & The set of candidate clients\\
		\hline
		$\mathcal{D}_{k}$   &Local dataset of client $k$\\
		\hline
		$\mathcal{D}_{k}^{A}$ & AIGC-enhanced dataset of client $k$\\
		\hline
		$F_{k,L}$  &Loss function of client $k$ over its local dataset \\
		\hline
		$F_{k,A}$  &Loss function for client $k$ over its AIGC-enhanced dataset\\
		\hline
		$p^{k}_{a}(y=i)$  &Client $k$ 's proportion of data samples for data synthesis\\
		\hline
		$d_{k}$ &Data size of client $k$\\
		\hline
		$\lambda_{k}$  & Parameter for data quality of client $k$'s local dataset\\
		\hline
		$\lambda_{k,A}$  &Parameter for data quality of client $k$'s AIGC-enhanced dataset\\
		\hline
		$s_{k}$  & Data-computing cost per unit data sample for client $k$\\
		\hline
		$\mathcal{N}$& The set of participating clients\\
		\hline
		$\mathcal{N}^{L}$& The set of participating clients using local datasets\\
		\hline
		$\mathcal{N}^{A}$& The set of participating clients using AIGC-enhanced datasets\\
		\hline
		$r$& The uniform unit data reward benchmark $r$ \\
		\hline
		$T$& The number of global iterations\\
		\hline
		$R(\mathcal{N}_{r})$& Total payment to participating clients given benchmark $r$ \\
		\hline
	\end{tabular}
\end{table}

\begin{figure*}[t]
	\begin{minipage}[t]{0.29\textwidth}
		\setlength{\abovecaptionskip}{-0.05cm}
		\centering
		\includegraphics[width=0.87\linewidth]{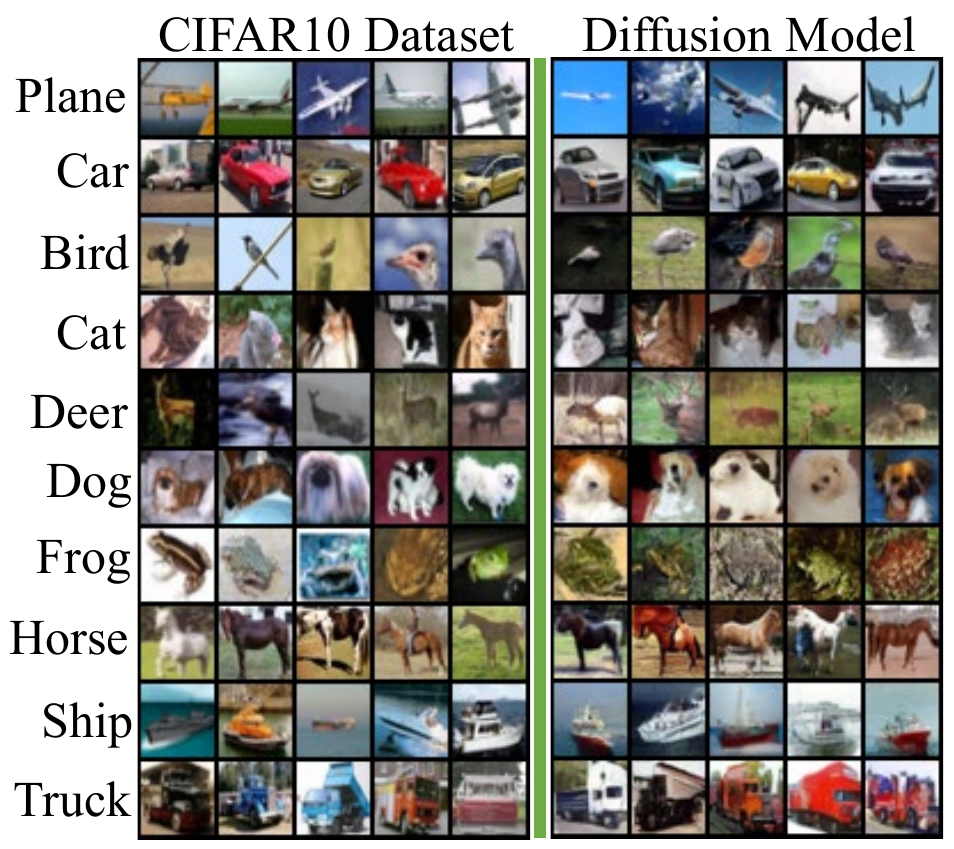}
		\caption{Comparison between real-world data samples and generated data samples.}
		\label{cifar_diffusion}
	\end{minipage}
	\begin{minipage}[t]{0.34\textwidth}
		\setlength{\abovecaptionskip}{-0.05cm}
		\centering
		\includegraphics[width=0.87\linewidth]{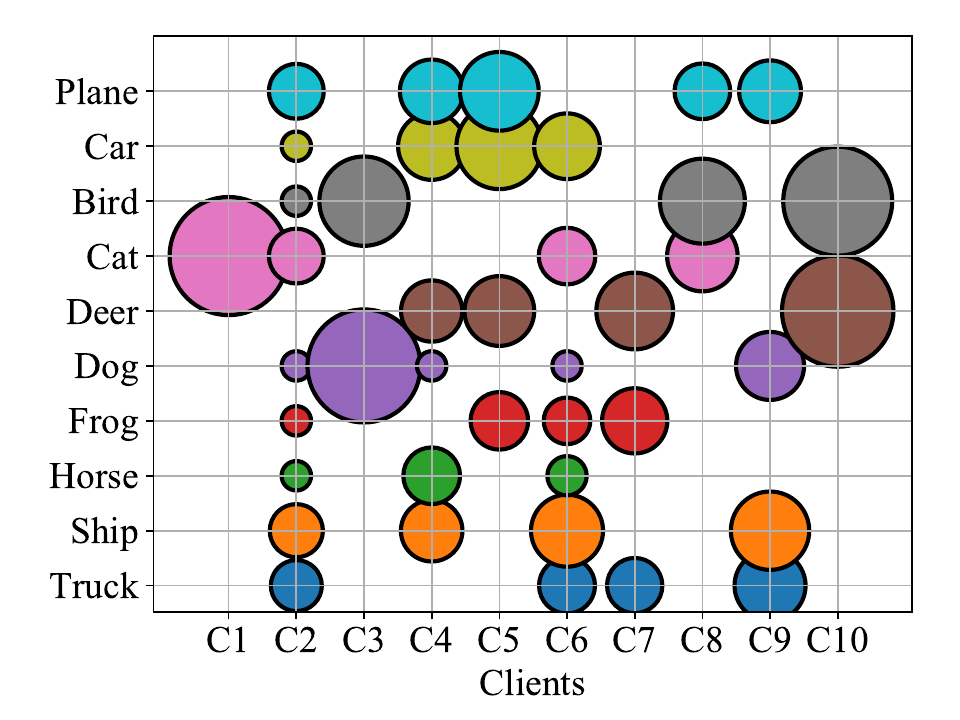}
		\caption{Distribution among classes of the local datasets for 10 randomly selected clients.}
		\label{local_dis}
	\end{minipage}
	\begin{minipage}[t]{0.34\textwidth}
		\setlength{\abovecaptionskip}{-0.05cm}
		\centering
		\includegraphics[width=0.87\linewidth]{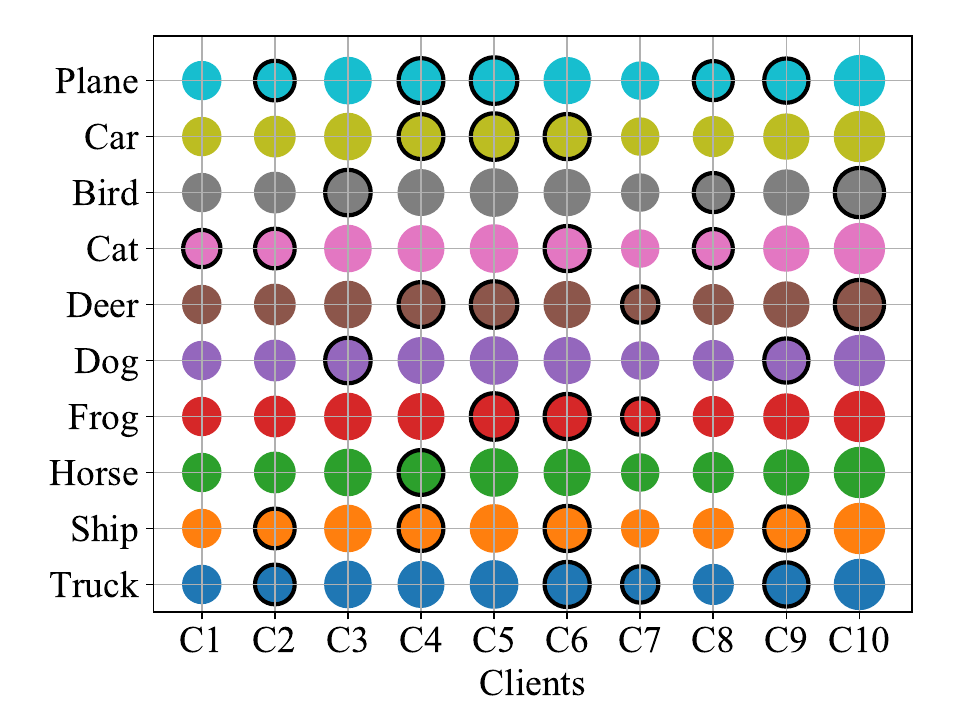}
		\caption{Distribution among classes of the AIGC-enhanced dataset for 10 randomly selected clients.}
		\label{arg_dis}
	\end{minipage}
\end{figure*}

For a general classification task, each data sample $(\boldsymbol{x},y)$ in the local dataset $\mathcal{D}_k$ distributes over $\mathcal{X} \times \mathcal{Y}$ following the distribution $\boldsymbol{p}$ \cite{2018non}, where $\mathcal{X}$ is a compact space of data features $\boldsymbol{x}$ and $\mathcal{Y}=\{1,...,Y\}$ represents the label space for ground-truth $y$. The number of total data samples of client $k$ is denoted as $d_k$. To conduct model training, a function (e.g., neural network) $f(\boldsymbol{x})=(f_{1}(\boldsymbol{x}),...,f_{Y}(\boldsymbol{x})): \mathcal{X}\rightarrow \mathcal{M}$ parameterized over the hypothesis class $\boldsymbol{w}$ is employed. Here, $f_{i}(\boldsymbol{x})$ quantifies the probability of data sample $(\boldsymbol{x},y )$ belonging to the $i$-th class and $\mathcal{M}=\{(m_{1},...,m_{Y}) \vert  \sum_{ i\in \mathcal{Y}}m_{i}=1,m_{i}\geq 0,\forall i\in \mathcal{Y}\}$. The loss function for client $k$ over its local dataset $\mathcal{D}_{k}$ can be defined with the widely-used cross-entropy loss as:
\begin{eqnarray}
\notag
F_{k,L}(\boldsymbol{w})=\mathbb{E}_{\boldsymbol{x}, y \sim \boldsymbol{p}}\left[\sum_{i=1}^{Y} -p^{k}(y=i) \log f_{i}(\boldsymbol{w}; \boldsymbol{x})\right]\\
=\sum_{i=1}^{Y} p^k(y=i)  \mathbb{E}_{x|y=i}\ell_{i}(\boldsymbol{w}; \boldsymbol{x}),
\end{eqnarray}
where $\ell_{i}(\boldsymbol{w}; \boldsymbol{x})=-\log f_{i}(\boldsymbol{w}; \boldsymbol{x})$. $p^{k}(y=i)$ denotes the proportion of data samples belonging to the $i$-th class in $\mathcal{D}_{k}$. Accordingly, the goal of FL is to minimize a global loss function, a weighted average of loss functions of all clients, i.e., $\min_{\boldsymbol{w}}F(\boldsymbol{w})= \sum_{k\in \mathcal{K}}\frac{d_{k}}{d(\mathcal{K})}
F_{k}(\boldsymbol{w})$,
where $d(\mathcal{K})=\sum_{k\in \mathcal{K}}d_{k}$ is the total number of data samples of all clients in $\mathcal{K}$.

Moreover, each client $k \in \mathcal{K}$ can also conduct data synthesis by adopting AIGC service (e.g., Dall-E or Stable Diffusion) to obtain a higher-quality AIGC-enhanced dataset $\mathcal{D}_{k}^{A}$, which   is a mixture of local real-world data samples and generated data samples under IID data distribution. As an illustration, we utilize diffusion model \cite{ho2020denoising} to mimic the provisioning of  AIGC service for the clients due to its impressive capability in
generating photo-realistic images with rich texture as shown in Fig. \ref{cifar_diffusion}. For client $k$ with non-IID data distribution, the data synthesis process is performed as follows: for each class $i$, given the proportion of data samples with class $i$ in the global data (denoted as $p(y=i)$) 
\footnote{In this paper, we consider global data follows IID distribution, since the server incentivizes numerous candidate clients in crowdsourcing platform, reaching a large number of data samples from each class.
Hence, we have $p(y=i)=\frac{1}{Y}$ for each class $i$ in IID data distribution, which can be used as the consensus of all clietns for data synthesis.} 
and the proportion of data samples with class $i$ in $\mathcal{D}_k$ (denoted as $p^{k}(y=i)$), the proportion of data samples for data synthesis is calculated as $p^{k}_{a}(y=i)=p(y=i)-p^{k}(y=i)$. 
For simplicity, we adopt this data synthesis process by keeping the datasize $d_k$ unchanged  \footnote{Such data synthesis process not only improves data quality (mimic IID distribution), but also saves large computing cost of local training at resource-constrained client devices incurred by generated samples from AIGC service. 
In reality, specifying a uniform data augmentation manner contributes to controlling the data quality of each client and predicting clients' behavior especially in incomplete information scenarios. 
Compared with training directly over the datasets only involving generated samples, in this paper the mixture of real-world and generated samples at clients' sides incurs lower cost for data synthesis and possesses higher data quality instead. We verify this fact by evaluating the data quality of AIGC-enhanced datasets in theoretical perspective and conducting experiments as elaborated in Section \ref{clients_att} and Section \ref{PEdis}, respectively}, i.e., $|\mathcal{D}^{A}_{k}|=|\mathcal{D}_{k}|=d_{k}$. Hence, $p^{k}_{a}(y=i)> 0$ indicates that client $k$ generates $p^{k}_{a}(y=i)d_{k}$ data samples with label $i$ by utilizing AIGC service. While client $k$ removes the redundant $p^{k}_{a}(y=i)d_{k}$ data samples with label $i$ from its local dataset when $p^{k}_{a}(y=i)\leq 0$. As depicted in Fig. \ref{local_dis}, we randomly select 10 clients with non-IID data distributions, the area of a circle represents the amount of data samples for a target class. By conducting data synthesis discussed above, we replenish data samples for the minority classes (circles without black borders as depicted in Fig. \ref{arg_dis}) and obtain higher-quality (e.g., more IID) AIGC-enhanced datasets for each client.

We assume that each generated data sample distributes over $\mathcal{X} \times \mathcal{Y}$ following distribution $\boldsymbol{p}_{a}$. %The loss function for client $k$ over its AIGC-enhanced dataset $\mathcal{D}_{k}^{A}$ is defined as:
By obtaining the definitions of $f_{i,a}(\boldsymbol{w};\boldsymbol{x})$ and $\ell_{i,a}(\boldsymbol{w};\boldsymbol{x})$ which are similar with $f_{i}(\boldsymbol{w};\boldsymbol{x})$ and $\ell_{i}(\boldsymbol{w};\boldsymbol{x})$, respectively, we define the loss function for client $k$ over its AIGC-enhanced dataset $\mathcal{D}_{k}^{A}$ as:
%We assume that each generated data sample distributes over $\mathcal{X} \times \mathcal{Y}$ following distribution $\boldsymbol{p}_{a}$. By obtaining definitions of $f_{i,a}(\boldsymbol{w};\boldsymbol{x})$ and $\ell_{i,a}(\boldsymbol{w};\boldsymbol{x})$ which are similar with $f_{i}(\boldsymbol{w};\boldsymbol{x})$ and $\ell_{i}(\boldsymbol{w};\boldsymbol{x})$ respectively, we define the loss function for client $k$ over its generated dataset $\mathcal{D}_{k}^{A}$ as:
%Similar to real-world data samples, we can obtain the similar definitions of $f_{i,a}(\boldsymbol{x})$ and $\ell_{i,a}(\boldsymbol{w}; \boldsymbol{x})$ corresponding to $f_{i}$ and $\ell_{i}$, respectively. Accordingly, we define the loss function for the client $k$ over its generated dataset $\mathcal{D}_{k}^{A}$ as:
\begin{eqnarray}\label{gen_loss}
\notag
F_{k,A}(\boldsymbol{w})
=& &\!\!\!\!\!\!\!\!\!\sum_{i\in \mathcal{I}_{k}} p^k_{a}(y=i)  \mathbb{E}_{x|y=i}\ell_{i,a}(\boldsymbol{w}; \boldsymbol{x})\\
& &\!\!\!\!\!\!\!\!\!\!+\sum_{i \notin \mathcal{I}_{k}} p(y=i)  \mathbb{E}_{x|y=i}\ell_{i}(\boldsymbol{w}; \boldsymbol{x}),
\end{eqnarray}
where $\mathcal{I}_{k}=\{i\vert i\in \mathcal{Y}, p_{a}^{k}(y=i)>0 \}$ denotes the set of labels injected by generated data samples for client $k$. %The definitions of $f_{i,a}(\boldsymbol{w};\boldsymbol{x})$ and $\ell_{i,a}(\boldsymbol{w};\boldsymbol{x})$ are similar with $f_{i}(\boldsymbol{w};\boldsymbol{x})$ and $\ell_{i}(\boldsymbol{w};\boldsymbol{x})$, respectively.

To reap the benefits from clients' local (or AIGC-enhanced) dataset, we adopt the widely-accepted synchronous FL paradigm, which entails multiple iterations between the server and clients for global model training until reaching a predetermined global iteration numbers $T$. Specifically, each iteration consists of both local model training process and global model aggregation process as follows: 
%We adopt the widely-accepted synchronous FL paradigm, which entails multiple global iterations between the server and clients. Each global iteration consists of both the local training process and the global model aggregation process. The server and clients iteratively engage in these global iterations until they reach a predetermined global iteration number, denoted as $T$. Specifically, a global iteration consists of local training and global iteration as follows:

1) Local model training on $\mathcal{D}_{k}$ or $\mathcal{D}_{k}^{A}$: The server incentivizes a set of clients $\mathcal{N}\subset \mathcal{K}$ to participate in FL.
In each global iteration $t$, each client $k\in \mathcal{N}$ receives current global model $\boldsymbol{w}^{t-1}$ from the server and conducts $h$-step local updates based on $\boldsymbol{w}^{t-1}$ through mini-batch stochastic gradient descent (SGD) algorithm over its local or AIGC-enhanced dataset. For ease of presentation, we denote $\boldsymbol{w}_{k}(t,t_{l})$ as local model of client $k$ at $(t_{l}-(t-1)h)$-th step in global iteration $t$. Here, $t_{l}\in ((t-1)h,th]$ represents the $t_l$-th local update starting from the beginning of global model training. According to the definition, the model parameters of client $k$ evolve as:
\begin{eqnarray}\label{thm1_eq1}
\boldsymbol{w}_{k}(t, t_{l})=\left\{
\begin{aligned}
&\boldsymbol{w}_{k}(t, t_{l}-1)-\eta \nabla F_{k}(\boldsymbol{w}_{k}(t, t_{l}-1);\boldsymbol{\xi}_{k}^{t,t_{l}}),
\\
&\quad \qquad \qquad\qquad\qquad\qquad \mbox{if}\  t_{l}\!\!\!\!\!  \mod h \! \neq\! 0,
\\
&\boldsymbol{w}(t, t_{l}),\  \mbox{if}\  t_{l} \!\!\! \mod h =0,
\end{aligned}
\right.
\end{eqnarray}
where 
\begin{eqnarray}\label{thm1_eq2}
\notag
\!\!\!\!\!\!\!\boldsymbol{w}(t, t_{l})=\sum_{k\in \mathcal{N}}\!\!\frac{d_{k}\!\!\left[ \boldsymbol{w}_{k}(t, t_{l}\!-\!1)\!-\!\eta \nabla F_{k}(\boldsymbol{w}_{k}(t, t_{l}\!-\!1);\boldsymbol{\xi}_{k}^{t,t_{l}})\right]}{d(\mathcal{N})}.\!\!\!\!\!\!\!\!\!\!\!\!\!\!\!\!\!\!\!\!\!\!\!\!\\
\end{eqnarray}
Here, $\nabla F_{k}(\boldsymbol{w}_{k}(t, t_{l}-1);\boldsymbol{\xi}_{k}^{t,t_{l}})$ denotes the gradient of client $k$ and $\boldsymbol{\xi}_{k}^{t,t_{l}}$ represents the mini-batch data samples. For ease of presentation, we have  $\nabla F_{k}(\boldsymbol{w}_{k}(t, t_{l}-1);\boldsymbol{\xi}_{k}^{t,t_{l}})=\nabla F_{k,L}(\boldsymbol{w}_{k}(t, t_{l}-1);\boldsymbol{\xi}_{k,L}^{t,t_{l}})$ for local dataset and $\nabla F_{k}(\boldsymbol{w}_{k}(t, t_{l}-1);\boldsymbol{\xi}_{k}^{t,t_{l}})=\nabla F_{k,A}(\boldsymbol{w}_{k}(t, t_{l}-1);\boldsymbol{\xi}_{k,A}^{t,t_{l}})$ for AIGC-enhanced dataset. Accordingly, we have $\boldsymbol{\xi}_{k}^{t,t_{l}}=\boldsymbol{\xi}_{k,L}^{t,t_{l}}$ for local dataset and  $\boldsymbol{\xi}_{k}^{t,t_{l}}=\boldsymbol{\xi}_{k,A}^{t,t_{l}}$ for AIGC-enhanced dataset. In this paper, we adopt a widely-used FL setting that the batchsize $B_{k}$ of each client $k$ is in the same proportion to its data size by setting $d_{k}=hB_{k}$ \cite{bourtoule2021machine, ding2020optimal, tran2019federated}. After $h$-step local updates, each client $k$ uploads its local model to the server.

%For ease of presentation, we use notation $\nabla F_{k}(\boldsymbol{w}_{k}(t, t_{l}-1);\boldsymbol{\xi}_{k}^{t,t_{l}})$ to denote the gradient of client $k$. Here, we have $\nabla F_{k}(\boldsymbol{w}_{k}(t, t_{l}-1);\boldsymbol{\xi}_{k}^{t,t_{l}})=\nabla F_{k,L}(\boldsymbol{w}_{k}(t, t_{l}-1);\boldsymbol{\xi}_{k,L}^{t,t_{l}})$ or $\nabla F_{k}(\boldsymbol{w}_{k}(t, t_{l}-1);\boldsymbol{\xi}_{k}^{t,t_{l}})=\nabla F_{k,A}(\boldsymbol{w}_{k}(t, t_{l}-1);\boldsymbol{\xi}_{k,A}^{t,t_{l}})$ corresponding to local or generated dataset, where $\boldsymbol{\xi}_{k}^{t,t_{l}}=\boldsymbol{\xi}_{k,L}^{t,t_{l}}$ or $\boldsymbol{\xi}_{k}^{t,t_{l}}=\boldsymbol{\xi}_{k,A}^{t,t_{l}}$ represents the mini-batch data samples from local or generated dataset in SGD algorithm at $[t_{l}-(t-1)h]$-step.  In this paper, we adopt a widely used FL setting that the batchsize of each client $B_{k}$ is the same proportion to its data size of dataset by setting $d_{k}=hB_{k}$. After $h$-step local updating, each client $k$ uploads its local model to the server. 

2) Global model aggregation: At global iteration $t$, the server receives all local models from client set $\mathcal{N}$, and conducts model aggregation to obtain a new global model $\boldsymbol{w}^{t} = \boldsymbol{w}(t, t_l)$ where $t_l\mod h = 0$.

\subsection{Client's Attributes}\label{clients_att}
Clients in FL scenarios greatly vary in device usage patterns (due to diverse physical characteristics and behavioral habits of users) and computation resource budgets (including computation capacity, memory, etc.), which results in diverse data-computing cost, data amounts and distributions among clients, and further affects FL model training performance, Hence, we identify each client $k$ by a tuple of attributes $\langle d_{k},  \lambda_{k},s_{k} \rangle$, each of which is characterized as follows:

%computing power differences

%The training performance of FL model and incentive cost depend on multi-dimensional attributes for each recruited client $k$, e.g., contributed datasize $d_{k}$ and data quality of local dataset $\lambda_{k}$ (i.e., non-iid degree), unit data-computing cost $s_{k}$. 

%diverse computing and storage capacities
%clients's different usage environments and patterns
%diverse physical characteristics and behavioral habits

%differ in hardware (CPU, memory)

%exhibit the kind of non-IID distributions

%Due to clients differing in devices (CPU, Memory) and usage habit, there exists inherent heterogeneity of data-computation cost $s_{k}$ , datasizes $d_{k}$ and data quality $\lambda_{k}$ for each client $k$. To quantify multi-dimensional attributes, we identify each client $k$ by a tuple of attributes $\langle d_{k},  \lambda_{k},s_{k} \rangle$, each of which is introduced as follows: 

%Consequently, each client $k$ can be identified by a tuple $\langle d_{k},  \lambda_{k},s_{k} \rangle$. We next introduce the definition for each attribute of client $k$.
%We utilize a tuple $\langle d_{k},  \lambda_{k},s_{k} \rangle$ to characterize client's multi-dimensional attributes, where $d_{k}$ is the total number of data samples of client $k$, $\lambda_{k}$ represents data quality (i.e, the non-IID degree) of client $k $ and $s_k$ is the unit data computing cost. As $d_k$ can be directly obtained by sending request to client $k$, we first calculate the value of $\lambda_{k}$ and $s_k$ in the following manner.

\subsubsection{Data quality of local and AIGC-enhanced datasets} We adopt  $\Vert \nabla F_{k,L}(\boldsymbol{w})-\nabla F (\boldsymbol{w})\Vert \leq \lambda_{k}$ for any $\boldsymbol{w}$ to characterize the data quality of client $k \in \mathcal{K}$, where $\lambda_{k}$ is defined as the upper bound of the gradient difference between the local loss function and the global loss function \cite{20DPkang}, with larger $\lambda_{k}$ being poorer data quality. Specifically, $\lambda_{k}$ can be estimated by the widely-adopted average earth mover's distance (EMD), which measures the data distribution heterogeneity (e.g., non-IID degree) among clients \cite{2018non}, i.e.,
\begin{eqnarray}
\label{lambda_local}
\notag
& &\Vert \nabla F_{k,L}(\boldsymbol{w})-\nabla F (\boldsymbol{w})\Vert\\
\notag
&=&\Vert \sum_{i=1}^{Y} [p^{k}(y=i)-p(y=i)] \nabla \mathbb{E}_{x|y=i}\ell_{i}(x,\boldsymbol{w}) \Vert\\
\notag
&\leq& \sum_{i=1}^{Y} \Vert p^{k}(y=i)-p(y=i) \Vert \Vert \mathbb{E}_{x|y=i}\ell_{i}(x,\boldsymbol{w})\Vert\\
&\leq&EMD_{k}\cdot g_{data}=\lambda_{k},
\end{eqnarray}
where $EMD_{k}= \sum_{i=1}^{Y} \Vert p^{k}(y=i)-p(y=i)\Vert$ and $g_{data}=\max_{i\in \mathcal{Y}}\Vert \mathbb{E}_{x|y=i}\ell_{i}(x,\boldsymbol{w})\Vert$ for any $\boldsymbol{w}$.

According to the definition in \eqref{gen_loss}, the data quality of AIGC-enhanced dataset $\mathcal{D}_{k}^{A}$ can be characterized by $\lambda_{k,A}$, which is defined as the upper bound of the gradient difference between the loss function calculated based on the AIGC-enhanced dataset and the global loss function:
%we aim to characterize the data quality of generated dataset $\mathcal{D}_{k}^{k}$ by an upper bound $\lambda_{k,A}$of difference between the local gradient of AIGC datasets and global gradient:
\begin{eqnarray}\label{lambda_ai}
\notag
& &\Vert \nabla F_{k,A}(\boldsymbol{w})-\nabla F (\boldsymbol{w})\Vert\\
\notag
&\leq& \sum_{i\in \mathcal{I}_{k}} p_{a}^{k}(y=i) \Vert\nabla \mathbb{E}_{x|y=i}\ell_{i,a}(\boldsymbol{w}; \boldsymbol{x})-\nabla \mathbb{E}_{x|y=i}\ell_{i}(\boldsymbol{w}; \boldsymbol{x})\Vert\\
&\leq&\frac{EMD_{k}}{2}\cdot g_{diff}= \lambda_{k,A},
\end{eqnarray}
%where $\mathcal{I}=\{i\vert i\in \mathcal{Y}, p_{a}^{k}(y=i)>0 \}$.
%Here, $p_{a}^{k}(y=i)>0$ means that the client $k$ should supply $d_{k}p_{ai}^{k}(y=i)$ generated data samples with the label $i$. In contrast, $p_{a}^{k}(y=i)\leq 0$ means that the client $k$ should remove $d_{k}p_{a}^{k}(y=i)$ from its dataset. 
where $g_{diff}=max_{i\in \mathcal{Y}}\Vert\nabla \mathbb{E}_{x|y=i}\ell_{i,a}(\boldsymbol{w}; \boldsymbol{x})-\nabla \mathbb{E}_{x|y=i}\ell_{i}(\boldsymbol{w}; \boldsymbol{x})\Vert$, for any $\boldsymbol{w}$, which characterizes the maximum gradient error between generated and real-world data samples among all classes. Intuitively, a lower $g_{diff}$ indicates that distribution of the generated data is more similar to real-world data, i.e., higher generated data quality.

%Note that $\lambda_k$ and $\lambda_{k,A}$ can be computed by the server at the end of the FL training process based on the uploaded gradient parameters of the clients. 

Based on the above discussions, we construct the relationship of the data quality between local dataset and AIGC-enhanced dataset for client $k$ based on the upper bounds, i.e., 
\begin{eqnarray}
\lambda_{k, A}=\theta\lambda_{k},
\end{eqnarray}
where $\theta=\frac{g_{diff}}{2g_{data}}$, and it is a constant given a fixed FL task and AIGC model.
In reality, the data quality of AIGC-enhanced dataset still depends on performance of AIGC model. We utilize $\theta$ to characterize the difference of data distribution between generated data samples (e.g., $\boldsymbol{p}_{a}$) and real-world data samples (e.g., $\boldsymbol{p}$). A lower $\theta$ indicates higher performance of generated data samples by AIGC (approaching data distribution of real-world data). Intuitively, $\theta<1$ means that introducing generated data can improve the data quality for each client \footnote{In this paper, we propose a gradient-based data quality evaluation method for both real-world and generated data samples. This evaluation approach (e.g., $\theta$) is closely related to classes of dataset (e.g., the label space $\mathcal{Y}$), allowing us to establish the relationship between the quality of local dataset and AIGC-enhanced dataset, which enables the derivation of a rational strategy for the server. 
Consequently, our technique can be readily extended to NLP tasks \cite{torfi2020natural}, such as sentiment analysis \cite{dos2014deep} and topic categorization \cite{johnson2015semi}. For more complex NLP tasks, a robust data quality evaluation is crucial for the effective incorporation of generated data samples in model training, which will be considered in future work.}. Compared with the data quality of an IID dataset only involving generated samples $g_{diff}$, the AIGC-enhanced dataset $\mathcal{D}_{k}^{A}$ possess higher data quality since $\lambda_{k,A}<g_{diff}$. 

\noindent \textbf{Remark} 1: In practice, the server can publish a small-scale public dataset (SPD) with IID data distribution on the crowd-sourcing platform, and then the value of $\lambda_k$ can be estimated by a client $k$ through recording the gradient differences on the initial model between its local dataset and the SPD as a testing process of model training.
Besides, $g_{data}$ can be estimated by recording maximum gradient norm based on the initial model for each class in SPD. The generated data quality $g_{diff}$ can be determined by recording the maximum gradient error for each class between samples in SPD and generated data samples by AIGC service. Then, parameter $\theta$ can be determined, which can be regarded as the consensus of all clients and the server.

%for client $k$ through once gradient calculation based on initial model before FL training. Here, PSD is datasets of a small amount of samples under IID data distribution provided by server. Besides, $g_{data}$ can be estimated by recording maximum gradient norm based on the initial model for each class in PSD. Besides, the generated data quality $g_{diff}$ can be determined by recording the maximum gradient error for each class between samples in PSD and generated samples (generated by AIGC service ). Then, $\theta=\frac{g_{diff}}{2g_{data}}$ will be reported to all candidate client and server by AIGC service provider.

%an upper bound-based relationship between the data quality of local dataset and generated dataset for client $k$, i.e., $\lambda_{k, A}=\theta\lambda_{k}$, where $\theta=\frac{g_{diff}}{2g_{data}}$ is a constant given a fixed FL task and generated AI model. Intuitively, $\theta<1$ means that introducing generated data can improve the data quality of local dataset for each client.

\subsubsection{Data-computing cost}
We denote the data-computing cost per unit data sample by $s_{k}$. As each client $k$ conducts $h$-step local updates with mini-batch size $B_{k}=\frac{d_{k}}{h}$ in one global iteration in our settings, the data-computing cost for local model training of client $k$ can be calculated as $d_{k}s_{k}$ in one global iteration. We assume that the stochastic gradient is unbiased and has a bounded variance over local dataset $\mathcal{D}_{k}$, i.e., for the mini-batch $\boldsymbol{\xi}_{k}\subset \mathcal{D}_{k}$, we have $\mathbb{E}\{\nabla F_{k,L}(\boldsymbol{w};\boldsymbol{\xi}_{k})\}=\nabla F_{k,L}(\boldsymbol{w})$ and $\Vert \nabla F_{k,L}(\boldsymbol{w};\boldsymbol{\xi}_{k})-\nabla F_{k,L}(\boldsymbol{w})\Vert^{2}\leq \frac{\psi^{2}}{B_{k}}$ \footnote{Different from some widely-used assumptions (e.g., $\mathbb{E}\{\nabla F_{k}(\boldsymbol{w};\boldsymbol{\xi}_{k})\}=\nabla F_{k}(\boldsymbol{w})$ and $\mathbb{E}\{\Vert \nabla F_{k}(\boldsymbol{w};\boldsymbol{\xi}_{k})-\nabla F_{k}(\boldsymbol{w})\Vert^{2}\} \leq \frac{\psi^{2}}{B_{k}}$)\cite{21DPmeetsFL}, \cite{2012optimalbatch}, we use the upper bound of $\Vert \nabla F_{k}(\boldsymbol{w};\boldsymbol{\xi}_{k})-\nabla F_{k}(\boldsymbol{w})\Vert^{2}$ in this paper.}. Here, $\psi$ is the sample variance, which is set to be a constant for all clients for simplicity. Intuitively, a larger $B_{k}$ (i.e., larger $d_{k}$) decreases the error of local gradient.

\subsection{Client's Utility Function}
To encourage clients' participation, we adopt a data quality-aware reward allocation mechanism. Given the uniform unit data reward benchmark $r$ published by the server, the final reward received by client $k$ is further discounted by its data quality in one global iteration, i.e., $rd_{k}(1-\frac{\lambda_{k}}{\lambda})$, where $\lambda$ is the largest non-IID degree tolerated by the server. Considering both the reward and data-computing cost, the utility of client $k$ in one global iteration when participating in FL with its local dataset $\mathcal{D}_{k}$ can be calculated as \footnote{Since the size of the uploaded model for each client is fixed, we neglect communication cost for simplicity.
}:
%In this paper, we propose a data quality-based crowdsourcing incentive mechanism. Specifically, the server publishes a uniform unit data reward for all clients denoted as $r$. The reward received by client $k$ is discounted by its data quality: $r(1-\frac{\lambda_{k}}{\lambda})$, where $\lambda$ is the largest non-iid degree tolerated by the server.  Hence, when the client $k$ participates in FL with its local datasets $\mathcal{D}_{k}$, its utility in one global iteration can be calculated as \footnote{In this paper, we assume that crowdsourcing platform possesses enough communication resources (bandwidth) for all clients. Since the size of the uploaded model for each client are fixed, we neglect communication cost for simplicity.}:
\begin{eqnarray}
U_{k}=rd_{k}(1-\frac{\lambda_{k}}{\lambda})- d_{k}s_{k}.
\end{eqnarray}

When client $k$ chooses to participate in FL with its AIGC-enhanced dataset with higher quality, it should further bear the cost of utilizing AIGC service for data synthesis \footnote{To generate AIGC-enhanced dataset, client $k$ needs to submit a custom order (data amount for different categories) to the AIGC service provider. In this paper, we assume AIGC service provider is a trusted third-party organization that charges for its services.}. Considering the communication and computing overheads and information hijacking risks (privacy-preserving overhead) when adopting AIGC service, we embrace a charge-per-utilization framework, wherein clients utilizing once-generated data are obligated to remit payment to AIGC service provider. We use $s_{AI}$ to denote the unit cost or payment of a client for one generated data sample (which can also factorize the overheads mentioned above). Hence, the payment for generated data samples in one global iteration is calculated as $d_{k}p_{k}^{+}s_{AI}$, where $d_{k}p_{k}^{+}=d_{k}\sum_{i\in\mathcal{I}_{k}}p^{k}_{a}(y=i)$, $\mathcal{I}_{k}=\{i\vert i\in \mathcal{Y}, p_{a}^{k}(y=i)>0 \}$. The utility of client $k$ in one global iteration when participating in FL with AIGC-enhanced dataset can be calculated as:
\begin{eqnarray}
U_{k}=rd_{k}(1-\frac{\lambda_{k, A}}{\lambda})- d_{k}s_{k}-d_{k}p_{k}^{+}s_{AI}.
\end{eqnarray}

%\footnote{Obtaining copyright ownership of the generated data samples from the AIGC service provider poses a significant financial burden and feasibility challenge for clients, given the exclusive utilization of these generated data samples for server's training alone. Compared with one-time payment, the charge-per-utilization method possesses a good scalability with clients needs.}

%client $k$ can obtain the reward $r(1- \frac{\lambda_{k,A}}{\lambda})$ when introducing generated data into its local dataset. Besides, the client should pay for AIGC service provider, when it choose to conduct data argumentation through generated data. In this paper, we embrace a charge-per-utilization framework, wherein clients utilizing once-generated data for local training are obligated to remit payment to the AIGC provider. Hence, when introducing generated data to its datasets, the payment to AIGC service provider of client $k$ can be calculated as $d_{k}p_{k}^{+}s_{AI}$, where $p_{k}^{+}=\sum_{i\in\mathcal{I}}p^{k}_{a}(y=i)$ and $s_{AI}$ is the unit payment of using one generated data sample\footnote{We adopt a charge-per-utilization in this paper. Due to heterogeneity of number of global iterations per client, charge-per-utilization enable each client to reduce its risk. Specifically, compared with one-time payment, this method possesses a good scalability with clients needs, which allows clients to start small and expand as their requirements evolve.}. Correspondingly, its utility in one global iteration can be expressed as:

\subsection{Server's Cost}
Aiming at obtaining a high-usaged model with low payments, the server strikes a trade-off between FL training performance (e.g., model accuracy loss) and payments to the participating clients, by minimizing the server's cost:% Hence, the server's cost can be formulated as:
\begin{eqnarray}\label{server_cost}
\mathcal{C}_{server}=\gamma_{1}M_{loss}+ \gamma_{2}R_{total},
\end{eqnarray}
where $\gamma_{1}$ and $\gamma_{2}$ balances the trade-off between model accuracy loss $M_{loss}$ and total payment $R_{total}$ to the clients.

%The server aims to obtain a high-usaged model with a low payment. Accordingly, the server's cost consists of payment to the participating clients and the FL model accuracy loss.

However, the calculation of model accuracy loss in AIGC-empowered FL faces great challenges, i.e., the heterogeneity of data quality and datasizes among clients, the mixture of different kinds of datasets (e.g., local or AIGC-enhanced dataset) used for local model updates due to clients' different choices. To address this issue, we  derive the convergence bound of the difference between the training loss in such AIGC-empowered FL scenario and optimal loss value over the local datasets of all clients. For ease of presentation, we define $F_{k}(\boldsymbol{w})=F_{k,L}(\boldsymbol{w})$ for local dataset and  $F_{k}(\boldsymbol{w})=F_{k,A}(\boldsymbol{w})$ for AIGC-enhanced dataset. For the purpose of convergence analysis, we first introduce the following widely-used assumptions on $F_{k}$ \cite{20DPkang}.
\begin{assume}\label{assum1} The training loss function satisfies the following properties:
	\begin{itemize}
		\item \textit{1)} $F_{k}(\boldsymbol{w})$ is $\beta$-Lipschitz, i.e., 	$\Vert  F_{k}(\boldsymbol{w})- F_{k}(\boldsymbol{w}') \Vert \leq \beta \Vert\boldsymbol{w}-\boldsymbol{w}'\Vert$ for any $ \boldsymbol{w}$ and $\boldsymbol{w}'$;
		\item \textit{2)} $F_{k}(\boldsymbol{w})$ is $\rho$-Lipschitz smooth, i.e., 	$\Vert \nabla F_{k}(\boldsymbol{w})- \nabla F_{k}(\boldsymbol{w}') \Vert \leq \rho \Vert\boldsymbol{w}-\boldsymbol{w}'\Vert$ for any $ \boldsymbol{w}$ and $\boldsymbol{w}'$;
		\item \textit{3)} $F_{k}(\boldsymbol{w})$ satisfies $\mu$-strong convex. Thus, $F_{k}(\boldsymbol{w})$ also satisfies Polyak-Lojasiewicz condition with parameter $\mu$, i.e., $F_{k}(\boldsymbol{w})-F_{k}(\boldsymbol{w}^{\ast})\leq \frac{1}{2\mu} \Vert\nabla F_{k}(\boldsymbol{w}) \Vert$ for any $\boldsymbol{w}$. Here, $\boldsymbol{w}^{\ast}$ is the optimal solution.
	\end{itemize}
\end{assume}
Due to limited space, we summarize the convergence analysis result as follows:
\begin{thm}\label{thm2}
	\textit{
		Based on Assumption \ref{assum1}, given the number of global iterations $T$, the set of participating clients is denoted as $\mathcal{N} \subseteq \mathcal{K}$ , i.e., $\mathcal{N}=\mathcal{N}^{L}\cup \mathcal{N}^{A}$. Here, $\mathcal{N}^{L}$ and $\mathcal{N}^{A}$ denote the set of participating clients with local datasets and AIGC-enhanced datasets, respectively. By setting $\eta< \frac{1}{\rho}$, the convergence upper bound is given as:	
		\begin{eqnarray}\label{coverge_1}
		F(\boldsymbol{w}(T, Th))-  F(\boldsymbol{w}^{\ast})\leq\phi^{hT}\Theta+(1-\phi^{hT})\kappa_{1} \Lambda(\mathcal{N}),
		\end{eqnarray}
		where
		\begin{small}
			\begin{eqnarray}
			\notag
			\Lambda(\mathcal{N})\!\!=\!\!\left[\sum_{k\in\mathcal{N}^{L}}\!\!\nu_{k}(\mathcal{N})\!\!\left(\frac{\psi}{\sqrt{B_{k}}}+\!\!\lambda_{k}\right)\!\!+\!\!\sum_{k\in\mathcal{N}^{A}}\!\!\nu_{k}(\mathcal{N})\!\!\left(\frac{\psi}{\sqrt{B_{k}}}+\lambda_{k,A}\!\!\right)\!\right].\!\!\!\!\!\!\!\!\!\!\!\!\!\!\!\!\\
			\end{eqnarray}
		\end{small}
		Here, $\Theta=F(\boldsymbol{w}(0, 0))-F(\boldsymbol{w^{\ast}})$, $\phi=1-2\mu\eta+2\mu\rho\eta^2$, $\kappa_{1}=\frac{\beta[(\eta\rho+1)^{h}-1]}{\rho(1-\phi^{h})}$, and  $\nu_{k}(\mathcal{N})=\frac{B_{k}}{\sum_{k\in \mathcal{N}}B_{k}}$. Here, $\boldsymbol{w}(0, 0)$ is the initial model parameters.}
\end{thm}

The proof is given in Subsection A in the separate supplementary file. 

\noindent\textbf{Remark} 2: $\Lambda(\mathcal{N})$ represents the error of training gradient caused by clients $\mathcal{N}$. The upper bound in \eqref{coverge_1} converges from $\Theta$ (e.g., when $T=0$) to $\kappa_{1}\Lambda(\mathcal{N})$ (e.g., when $T \rightarrow \infty$) with the increase of the number of global iterations $T$.
Hence, we assume that $\Theta>\kappa_{1}\Lambda(\mathcal{N})$ in mathematics due to the fact that FL model training with multiple iterations improves the model performance and leads to a lower convergence upper bound.
Furthermore, the value of $\Lambda(\mathcal{N})$ decreases as the amount of data from each participating client increases and as the value of $\lambda_{k}$ (or $\lambda_{k, A}$) from each participating client.
This indicates that the server can enhance the final model performance by selecting clients with substantial data volumes and dataset distributions closely approximating the global data distribution (indicated by lower value of $\lambda_{k}$ or $\lambda_{k, A}$).

Based on Theorem \ref{thm2}, we use the right-hand side of \eqref{coverge_1} to represent the model accuracy loss:
\begin{eqnarray}\label{M_loss}
\notag
M_{loss}
\!\!\!&=&\!\!\!\phi^{hT}\Theta + (1-\phi^{hT})\kappa_{1}\left[\sum_{k\in\mathcal{N}}\frac{\psi\sqrt{d_{k}}}{\sqrt{h}d(\mathcal{N})}\right.\\
\!\!\!& &\!\!\!+
\left.\sum_{k\in\mathcal{N}^{L}}\frac{d_{k}}{d(\mathcal{N})}\lambda_{k}+\sum_{k\in\mathcal{N}^{A}}\frac{d_{k}}{d(\mathcal{N})}\lambda_{k,A}\right],
\end{eqnarray}
where $d(\mathcal{N})=\sum_{k\in \mathcal{N}}{d}_{k}$. The equation \eqref{M_loss} originates from $d_{k}=hB_{k}$.
For simplicity, the number of steps for local updates $h$ is set as a constant in this paper.

Prior to commencing with FL model training, the server incentivizes a subset of clients $\mathcal{N}=\mathcal{N}^{L}\cup \mathcal{N}^{A}$ with a uniform unit data reward benchmark $r$. Hence, the total payment to clients after one global iteration is calculated as: $R_{pay}=\sum_{k\in \mathcal{N}^{L}}rd_{k}(1-\frac{\lambda_{k}}{\lambda})+\sum_{k\in \mathcal{N}^{A}}rd_{k}(1-\frac{\lambda_{k,A}}{\lambda})$. As a result, the total payment after the overall FL training process is %calculated as
$R_{total}=TR_{pay}$.

\subsection{Data Quality-Aware Incentive Mechanism}
To minimize the server's cost defined in \eqref{server_cost}, we introduce a crowd-sourcing platform for AIGC-empowered FL scenario and devise a data quality-aware incentive mechanism for client recruitment. Specifically, the server first publishes its training task including the relevant small-scale public dataset to the crowd-sourcing platform. Based on acquired common information, the server publishes the number of global iterations $T$ and the uniform unit data reward benchmark $r$ to incentivize clients' participation. In response to server's strategy $(T,r)$, each client $k$ should determine its strategy \footnote{In this paper, we focus on incentive mechanism design for AIGC-empowered FL scenario. Consequently, we assume that each client is selfish, but has no malicious intention (e.g., inject noise). As for model security issue, the server can adopt a lightweight method to check the quality of models uploaded by the clients through a validation dataset to identify potential malicious behaviors \cite{zhang2022robust}.}, i.e., whether to participate in FL and which dataset (e.g., local or AIGC-enhanced dataset) to be utilized.

However, each client may be reluctant to disclose its private attributes (e.g., data quality and data-computing cost) to the server prior to joining in FL in practices, which poses great challenge for the incentive strategy making of the server due to information asymmetry. In light of this, we study the optimal incentive strategy for the server under different information settings as follows:
\begin{itemize}
	\item Complete information scenario: 
	The server knows the multi-dimensional attributes $(d_{k},\lambda_{k},s_{k})$ of each client $k$, which can also serve as the benchmark and provide insights for more practical scenario of incomplete information.
	
	\item Incomplete information scenario: 
	The server does not know clients' actual data quality and unit data-computing cost ($\lambda_{k}$, $s_{k}$) for incentive mechanism decision making beforehand, but it knows the value of $d_{k}$ and the probability distributions of $\lambda_{k}$ and $s_{k}$ for each client $k$	\footnote{Following most existing studies \cite{ding2020optimal, ding2020incentive}, we assume that the server 
	is aware of the contributed datasize of each participating clients for global weighted aggregation, while is not able to reach any other information about clients' local datasets.
		Besides, for the decision making, the server can obtain distribution information of data quality and data-computing cost through market research \cite{ding2020optimal}. Although the clients do not disclose its private attribute of data quality before joining the FL in the incomplete information scenario, the server can still determine the payments based on the uploaded model gradient parameters (for computing clients' data quality) during the first round of FL training in practice.}. 
	%In incomplete information scenario, the crowdsourcing platform acts as a trusted third-part platform to finish the implementation of reward allocation.
	%Specifically, the crowdsourcing platform measures the actual data quality of local dataset for each client (e.g., $\lambda_{k}$) based on SPD, and reports $\lambda_{k}$ and $\lambda_{k,A}=\theta\lambda_{k}$ to each client before FL model training to facilitate clients' decisions. After FL model training process, crowdsourcing platform determines the final payment to the participating clients and allocates rewards based on their data quality and unit data reward benchmark $r$ announced by the server.

	%We note that as discussed above, 
	%the actual data quality of datasets used by each client ($\lambda_{k}$, $\lambda_{k,A}=\theta\lambda_{k}$) is measured by crowdsourcing platform, which will be reported to clients before FL training and used for determining the final payment to the participating clients accordingly.

\end{itemize}

%\begin{itemize}
%	\item Complete information scenario: The server knows the multi-dimensional attributes $(d_{k},\lambda_{k},s_{k})$ of each client $k$. This can serve as the benchmark and provide insights for more practical scenario of incomplete information.
%	\item Incomplete information scenario: The server does not know clients' actual data quality and unit data-computing cost ($\lambda_{k}$, $s_{k}$) for incentive mechanism decision making beforehand, but it knows the value of $d_{k}$ and the probability distributions of $\lambda_{k}$ and $s_{k}$ for each client $k$ \footnote{Similar to most existing studies \cite{ding2020optimal, ding2020incentive}, we assume that each client needs to report its contributed datasize to the server. Besides, for the decision making, the server can obtain distribution information of data quality and data-computing cost through market research \cite{ding2020optimal}. Also, we note that as discussed above, the actual data quality of a client can be measured after the FL training, which will be used for determining the final payment to the client accordingly.}.
%\end{itemize}

\section{Complete information scenario}
In this section, we first study the clients' behaviors, and then derive the optimal strategy for the server to minimize the server's cost in complete information scenario.

\subsection{Client's Behavior}
The strategy of client $k$ can be defined by $\varphi_{k}(\varphi_{k,p}, \varphi_{k,a})$, where $\varphi_{k,p}$ and $\varphi_{k,a}$ indicate whether to participate in FL and whether to use generated data samples for local model updates, respectively. Hence, we define $(\varphi_{k,p}, \varphi_{k,a}) \in \{(0,0),(1,0),(1,1)\}$. As rational clients always adopt a strategy to maximize their utilities, each client $k$ may choose to not participate in FL (i.e., $\varphi_{k}(0,0)$) when its largest utility is $U_{k}(\varphi_{k}(0,0))=0$. For client $k$ which chooses to participate in FL, it conducts local model updates with local dataset when $U_{k}(\varphi_{k}(1,0)) > U_{k}(\varphi_{k}(1,1))$, and uses AIGC-enhanced dataset when $U_{k}(\varphi_{k}(1,1)) \geq U_{k}(\varphi_{k}(1,0))$. Note that for simplicity, in this study we consider the batch data generation mode such that if a client decides to use its original local dataset for FL, it will not generate new dataset; otherwise, it will adopt the AIGC service to generate a higher-quality dataset before joining FL. We will further consider the  more complicated  iterative data generation mode in future study, such that a client can gradually and iteratively generate data samples (which is time-consuming and would slow down the FL operations) to improve its dataset during FL process.

Given the values of client's attributes $\langle d_{k},\lambda_{k},s_{k}\rangle$ and the unit data reward benchmark $r$, we derive the following conclusion based on clients' rational behaviors:

%In this subsection, we study the rational behavior for each client given the unit data reward $r$ announced by the server.

%Specifically, the strategy of each client $k$ can be identified by a tuple $\varphi_{k}(\varphi_{k,p}, \varphi_{k,a})$. where $\varphi_{k,p}=1$ indicates that client $k$ chooses to participate in FL. Conversely, $\varphi_{k,p}=0$ represents that client $k$ does not participate in FL. When the client $k$ participates in FL $\varphi_{k,p}=1$, $\varphi_{k,a}=0$ and $\varphi_{k,a}=1$ imply that the client $k$ participates in FL with local datasets $\mathcal{D}_{k}$ and new generated datasets $\mathcal{D}_{k}^{A}$, respectively. In a word, the strategies of client $k$ includes $\varphi_{k}(0,0)$, $\varphi_{k}(1,0)$ and $\varphi_{k}(1,1)$.

%A rational client $k$ will adopt a strategy to maximize its utility. Intuitively, the utility of client $k$ is zero when it does not participate in FL, i.e., $U_{k}(\varphi_{k}(0,0))=0$. We need to point out that the client $k$ will choose to participate in FL with generated dataset $\mathcal{D}_{k}^{A}$, when $0 \leq U_{k}(\varphi_{k}(1,0))\leq U_{k}(\varphi_{k}(1,1))$.

%Based on the clients' rational behaviors, we derive the following conclusion:

\begin{thm}\label{thm3}
	\textit{
		All clients can be divided into two types: \textit{type-1} and \textit{type-2} clients, whose behaviors can be revealed by three indicators: $\zeta_{1}(k)=\frac{\lambda s_{k}}{\lambda-\lambda_{k}}$, $\zeta_{2}(k)=\frac{\lambda s_{k}+\lambda \delta  \lambda_{k}}{\lambda-\theta \lambda_{k}}$, and $\zeta_{3}=\frac{\lambda \delta }{1-\theta}$, where $\delta=\frac{s_{AI}}{2g_{data}}$. Specifically, the set of \textit{type-1} clients can be represented by $\mathcal{T}_{1}=\{k\in \mathcal{K}| \zeta_{1}(k)\leq \zeta_{3}\}$ and the set of \textit{type-2} clients can be denoted as $\mathcal{T}_{2}=\{k\in \mathcal{K}| \zeta_{3}<\zeta_{1}(k)\}$\footnote{Note that, we have $\zeta_{1}(k)<\zeta_{2}(k)<\zeta_{3}$ and  $\zeta_{3}<\zeta_{2}(k)<\zeta_{1}(k)$ for \textit{type-1} and \textit{type-2} clients, respectively. When $\zeta_{1}(k)=\zeta_{3}$, we must have $\zeta_{1}(k)=\zeta_{2}(k)=\zeta_{3}$.}.}
	\textit{
		Given the uniform unit data reward benchmark $r$, the strategy of \textit{type-1} clients can be represented as:
		\begin{eqnarray}\label{client_type1}
		\varphi_{k}=\left\{
		\begin{aligned}
		&\varphi_{k}(0,0), \qquad \mbox{if} \ r<\zeta_{1}(k),\\
		&\varphi_{k}(1,0), \qquad    \mbox{if} \ \zeta_{1}(k)\leq r< \zeta_{3},\\
		&\varphi_{k}(1,1), \qquad  \mbox{if} \ \zeta_{3}\leq r.
		\end{aligned}
		\right.
		\end{eqnarray}
		Similarly, the strategy of \textit{type-2} clients can be represented as:
		\begin{eqnarray}\label{client_type2}
		\varphi_{k}=\left\{
		\begin{aligned}
		&\varphi_{k}(0,0), \qquad \mbox{if} \ r<\zeta_{2}(k),\\
		&\varphi_{k}(1,1), \qquad    \mbox{if} \ \zeta_{2}(k)\leq r.\\
		\end{aligned}
		\right.
		\end{eqnarray}}
\end{thm}

The proof is given in Subsection B in the separate supplementary file. 

\noindent \textbf{Remark} 3: Fig. \ref{clientbehave} summarizes the rational behavior for each type of clients. Intuitively, \textit{Type-1} clients possess local datasets with higher data quality (e.g., lower $\lambda_{k}$) and lower data-computing cost. While \textit{type-2} clients can not directly benefit from FL process with their local datasets due to lower data quality. Given a larger data reward benchmark $r$ published by the server, \textit{type-2} clients reap the benefits of participating in FL by using AIGC-enhanced dataset for local model training.

\begin{figure}[t]
	\centering
	\includegraphics[width=0.95\linewidth]{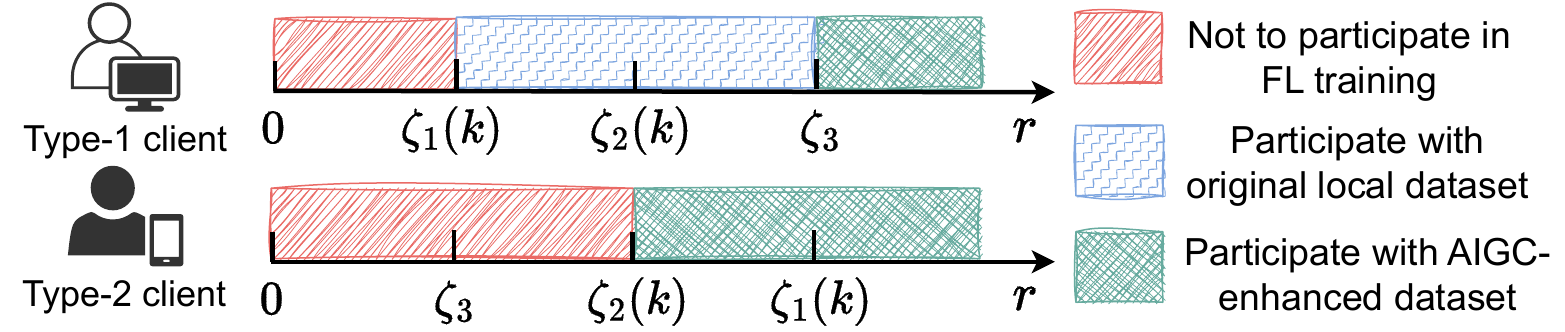}
	\vspace{-5pt}
	\caption{The rational behavior of each client $k$.}
	\label{clientbehave}
	\vspace{-10pt}
\end{figure}

%When the server publishes a large uniform reward $r$, \textit{type-2} clients can benefit from participating in FL through introducing generated data.

\begin{figure*}[htbp]
	\centering
	\includegraphics[width=0.86\linewidth]{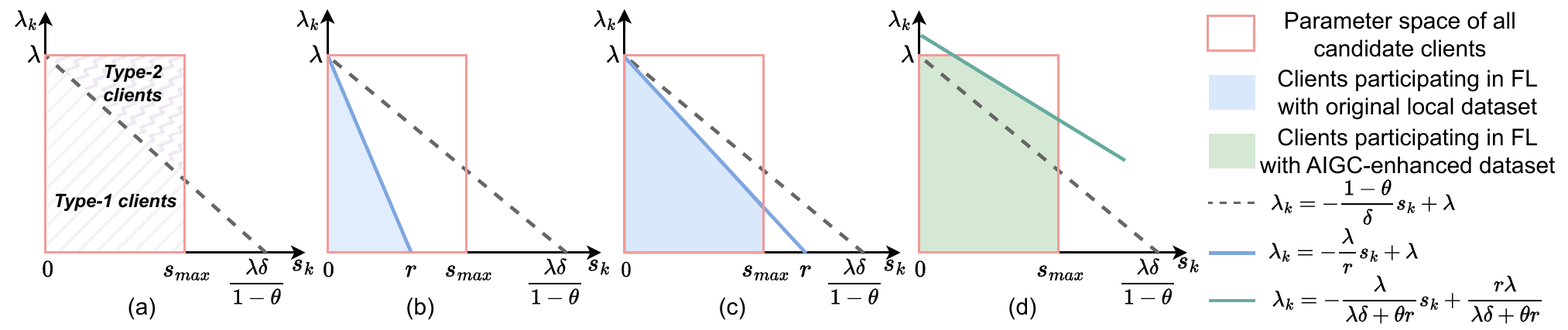}
	\vspace{-10pt}
	\caption{Parameter space of $(s_{k},\lambda_{k})$ for clients.}
	\vspace{-8pt}
	\label{figincomplte}
\end{figure*}

%Specifically, the server should determine the total global iteration $T$ and then it needs to determine the uniform reward $r$ considering clients' behaviors. 

\subsection{Server's Optimal Strategy in Complete Information Scenario}
Considering clients' rational behaviors, the server aims to minimize its cost in \eqref{server_cost} by finding the optimal strategy:
%The server aims to minimize its cost in \eqref{server_cost}, considering clients' rational behaviors:
\begin{align}
&\mathop{Minimize}\limits_{T, r} \  \mathcal{C}_{cost},\label{Scost_problem}
\\ &\mbox{s.t.}\quad T\in \{1,2,...\}, \tag{\ref{Scost_problem}{a}} \label{Scost_problema}
\\ &  \quad \quad \,  r \in (0,+\infty) \tag{\ref{Scost_problem}{b}} \label{Scost_problemb}.
\end{align}
It is challenging to directly solve problem \eqref{Scost_problem} due to the mutual coupling of the number of global iterations $T$ and the unit data reward benchmark $r$. 
In what follows, we first derive some useful insights for $r$ under any fixed $T$, based on which, we derive the optimal server's strategy.

According to clients' rational behaviors, a uniform unit data reward benchmark $r$ corresponds to a unique participating client set $\mathcal{N}_{r}=\mathcal{N}_{r}^{L}\cup \mathcal{N}_{r}^{A}$. We denote the training error during FL process caused by clients $\mathcal{N}_r$ as $\Lambda(\mathcal{N}_{r})=\sum_{k\in \mathcal{N}_{r}^{L}}\nu_{k}(\mathcal{N}_{r})\chi_{k}+\sum_{k\in \mathcal{N}_{r}^{A}}\nu_{k}(\mathcal{N}_{r})\chi_{k, A}$, where $\chi_{k}=(\frac{\psi}{\sqrt{B_{k}}}+\lambda_{k})$ and $\chi_{k,A}=(\frac{\psi}{\sqrt{B_{k}}}+\lambda_{k, A})$.
Hence, given a fixed number of global iterations $T$, the server's cost can be represented as:
\begin{eqnarray}\label{T_server_cost}
\notag
\mathcal{C}_{server}(T, r)=\!\!\!\!\!\!& &\!\!\!\!\!\!\gamma_{1}
\phi^{hT}\Theta\\
\!\!\!\!\!\!& &\!\!\!\!\!\!+(1-\phi^{hT})\gamma_{1}\kappa_{1}\Lambda(\mathcal{N}_{r})+\gamma_{2}TR(\mathcal{N}_{r}),
\end{eqnarray}
where $R(\mathcal{N}_{r})=\sum_{k\in \mathcal{N}_{r}^{L}}d_{k}(1-\frac{\lambda_{k}}{\lambda})+\sum_{k\in \mathcal{N}_{r}^{A}}d_{k}(1-\frac{\lambda_{k,A}}{\lambda})$, which indicates the payment to clients corresponding to unit data reward benchmark $r$ in one global iteration.
We next reveal the property of the optimal unit data reward benchmark $r_{o}$:

\begin{prop}\label{pro2}
	\textit{
		In complete information scenario, for any fixed number of iterations $T$, the optimal unit data reward benchmark $r_{o}$ to minimize the server's cost in \eqref{T_server_cost} must belong to the set of all clients' indicators $\boldsymbol{\zeta}$, i.e.,
		\begin{eqnarray}\label{zeta_set}
		r_{o}\in \boldsymbol{\zeta}=\bigcup_{k\in\mathcal{K}}\{\zeta_{1}(k), \zeta_{2}(k), \zeta_{3}\}.
		\end{eqnarray}}
\end{prop}
\begin{proof}
	We set $\boldsymbol{\zeta}\gets \boldsymbol{\zeta} \cup \{\infty\}$ and sort indicators in $\boldsymbol{\zeta}$ in ascending order.
	Assume that $r_{o}\notin \boldsymbol{\zeta}$, we can search two neighboring indicators $r_{1} \in \boldsymbol{\zeta}, r_{2}\in \boldsymbol{\zeta}$ such that $r_{1}<r_{o}<r_{2}$. Then, we have $\Lambda(\mathcal{N}_{r_{1}})=\Lambda(\mathcal{N}_{r_{o}})$, since each client's strategy under reward $r_{1}$ is the same as that under reward $r_{o}$. Furthermore, we have $R(\mathcal{N}_{r_{1}})<R(\mathcal{N}_{r_{o}})$ due to $r_{1}<r_{o}$. As a result, we have $\mathcal{C}_{server}(T,r_{1})<\mathcal{C}_{server}(T,r_{o})$, which contracts with the assumption. Thus, we finish the proof of Proposition \ref{pro2}.  
\end{proof}

Based on Proposition \ref{pro2}, we next reveal the property of the optimal number of global iterations given a fixed unit data reward benchmark $r$:

\begin{prop}\label{pro3}
	\textit{Given a fixed unit data reward benchmark $r$ and the assumption $\Theta>\kappa_{1}\Lambda(\mathcal{N}_{r})$, the server's cost in \eqref{Scost_problem} is a convex function respect to $T$. By setting the first-order derivative $\frac{\partial \mathcal{C}_{cost}}{\partial T}$ to $0$, the optimal number of global iterations $T^{\ast}(r)$ can be calculated as:
		\begin{eqnarray}\label{pro3_eq}
		T^{\ast}(r)=\log_{\phi^h}\left[\frac{\gamma_{2}R(\mathcal{N}_r)}{\gamma_{1}(-\ln \phi^h)(\Theta-\kappa_{1}\Lambda(\mathcal{N}_r))}\right].
		\end{eqnarray}}
\end{prop}
\textbf{Remark} 4: The assumption that $\Theta>\kappa_{1}\Lambda(\mathcal{N}_{r}) $ is reasonable since FL model training improves the model performance. We can see that the optimal $T^{\ast}(r)$ is decreasing with the increases of $R(\mathcal{N}_{r})$ and $\Lambda(\mathcal{N}_{r})$, since $\phi^{h}<1$.

Based on above analysis, we can obtain the server's strategy profile $(T^{\ast}(r), r)$ for each candidate $r$ in $\boldsymbol{\zeta}$ set. By calculating the server' s cost for each strategy profile, we choose the optimal one which has the minimum server's cost. The above procedure to obtain server's optimal strategy in complete information scenario is summarized in Algorithm \ref{completeagl}. For computational efficiency, Algorithm \ref{completeagl} can obtain the optimal strategy for the server with complexity of $O(2|\mathcal{K}|)$.

\begin{algorithm}[t]
	\small
	\caption{Server's optimal strategy in complete information scenario}
	\label{completeagl}
	\begin{algorithmic}[1]
		\Require The candidate client set $\mathcal{K}$.
		\Ensure Server's optimal strategy profile $(T_{o}, r_{o})$.
		\State $\mathcal{S} \gets \emptyset$.
		\For {each $r \in \zeta $}
		\State Calculate $T^{\ast}(r)$ based on \eqref{pro3_eq} and obtain the strategy profile $\mathcal{S}\gets \mathcal{S}\cup \{(T^{\ast}(r),r) \}$.
		%\State Calculate server's cost $\mathcal{C}_{server}(T^{\ast}(r),r)$.
		\EndFor
		\State Search the optimal strategy profile in $\mathcal{S}$ such that it minimizes server's cost $\mathcal{C}_{server}$, which is denoted as $(T_{o},r_{o})$.
		\State \textbf{return} strategy profile $(T_{o},r_{o})$
	\end{algorithmic}
\end{algorithm}

\section{Incomplete Information Scenario}
In incomplete information scenario, the server is only aware of the distributions of data quality $\lambda_{k}$ and the unit data usage cost $s_{k}$, instead of knowing the exact values for client's attributes $(\lambda_{k},s_{k})$. Hence, we aim to derive the server's optimal strategy to minimize server's expected cost.

\subsection{Clients' Behaviors in Incomplete Information Scenario}
\label{subsec_client}
Without loss of generality, we assume that  $\lambda_{k}\in (0,\lambda_{max}]$ and $s_{k}\in (0,s_{max}]$ hold for all clients, where $\lambda_{max}=\lambda$ \footnote{In this paper, we set the non-iid degree  $\lambda$ tolerated by the server equaling to $\lambda_{max}$, the maximum value of $\lambda_{k}$. Our analysis can be easily extended to the case when $\lambda<\lambda_{max}$.} and $s_{max}<\zeta_{3}=\frac{\lambda \delta}{1-\theta}$ \footnote{This assumption originates from the fact $\lambda> 1$ in real-world datasets (e.g., MNIST and CIFAR10 datasets). The maximum unit data-computing cost satisfy $s_{max}<<1$ in our experimental settings. Our analysis can be extended in cases where $s_{max}\geq \zeta_{3}$.}. In incomplete information scenario, the server is aware of the probability density functions of $\lambda_{k}$ and $s_k$, which are denoted as $u(\cdot)$ and $v(\cdot)$, respectively. Here,  $\lambda_{k}$ and $s_{k}$ are independent with each other.

%In this subsection, we aim to analyze the expected strategies for clients in incomplete information scenario, given the unit data reward $r$.

%Specifically, we assume that the $\lambda_{k}\in (0,\lambda_{max}], \lambda_{max}=\lambda$ and $u(\cdot)$ is the probability density function of $\lambda_{k}$ \footnote{In this paper, we set the non-iid degree  $\lambda$ tolerated by the server equaling to $\lambda_{max}$, the maximum value of $\lambda_{k}$. Our analysis can be easily extended to the case when $\lambda<\lambda_{max}$.}. Similarly, the probability density function of $s_{k}\in (0,s_{max}]$ is denoted as $v(\cdot)$. $\lambda_{k}$ and $s_{k}$ are independent of each other. For simplicity, we assume that $s_{max}<\zeta_{3}=\frac{\lambda \delta}{1-\theta}$\footnote{This assumption originates from the fact $\lambda> 1$ in real datasets (MNIST and CIFAR10). The maximum unit data usage cost satisfy $s_{max}<<1$ in our experimental settings. Our analysis can be extended in cases where $s_{max}\geq \zeta_{3}$.}.

To intuitively show clients' behaviors associated with their personal attributes, a two-dimensional parameter space is constructed for all clients based on $s_{k}$ and $\lambda_{k}$. As depicted in Fig. \ref{figincomplte}(a), all clients fall in the parameter space $D_{space}=\{(s_{k},\lambda_{k})\vert s_{k}\in(0, s_{max}], \lambda_{k}\in (0,\lambda] \}$ following the joint probability distribution of $s_{k}$ and $\lambda_{k}$. Obviously, \textit{type-1} and \textit{type-2} clients are divided by the dotted line $\lambda_{k}=-\frac{1-\theta}{\delta}s_{k}+\lambda$. Given the uniform unit data reward benchmark $r$, we can derive the probability of each strategy adopted by a client by considering the following three cases:
%Fig. \ref{figincomplte} illustrates the two-dimensional parameter space formed by $s_{k}$-axes and $\lambda_{k}$-axes, where each client randomly falls onto the parameter space according to the probability distribution.  Obviously, \textit{type-1} and \textit{type-2} clients are divide by the dotted line. Next, we derive the probability of each strategy adopted by a client given the uniform unit data reward $r$. 
\begin{itemize}
	\item Case \textbf{(\textit{i})} ($r\in (0, s_{max})$): According to clients' rational behaviors (See Theorem \ref{thm3}), each \textit{type-1} client $k$ will not participate in FL with AIGC-enhanced dataset (i.e., strategy $\varphi_{k}(1,1)$) due to $r<s_{max}<\zeta_{3}=\frac{\lambda \delta}{1-\theta}$. When $r\geq \zeta_{1}(k)\Rightarrow \lambda_{k}\leq -\frac{\lambda}{r}s_{k}+\lambda$, each \textit{type-1} client $k$ will participate in FL with its local dataset. 
	As for each \textit{type-2} client $k$, it will not participate in FL since $r<\zeta_{3}<\zeta_{2}(k)$.  Accordingly, as illustrated in Fig. \ref{figincomplte}(b), 
	client whose attribute tuple $(s_{k},\lambda_{k})$ falls within the blue region will participate in FL with its local dataset. Thus, the probability that client $k$ participates in FL can be calculated as $P_1=\iint\limits_{D_{1}} u(s_{k})v(\lambda_{k}){\rm d}s_{k}{\rm d}\lambda_{k}$,
	%\begin{eqnarray}
	%P_1=\iint\limits_{D_{1}} u(s_{k})v(\lambda_{k}){\rm d}s_{k}{\rm d}\lambda_{k},
	%\end{eqnarray}
	where $D_{1}=\{(s_{k},\lambda_{k})\vert s_{k}\in(0, r), \lambda_{k}\in (0,-\frac{\lambda}{r}s_{k}+\lambda) \}$ represents the blue region in Fig. \ref{figincomplte}(b).
	
	\item Case \textbf{(\textit{ii})} ($r\in [s_{max},\zeta_{3})$): Similar to the above case, only \textit{type-1} clients (i.e., the clients located in the blue region $D_{2}=\{(s_{k},\lambda_{k})\vert s_{k}\in(0, s_{max}), \lambda_{k}\in (0,-\frac{\lambda}{r}s_{k}+\lambda)\}$ in Fig. \ref{figincomplte}(c)) will participate in FL with its local dataset. Accordingly, the probability that client $k$ participates in FL can be denoted as $P_2=\iint\limits_{D_{2}} u(s_{k})v(\lambda_{k}){\rm d}s_{k}{\rm d}\lambda_{k}.$
	
	\item Case \textbf{(\textit{iii})} ($r\in [\zeta_{3},\frac{\lambda\delta+s_{max}}{1-\theta}]$)\footnote{	The decision of each client under $r>\frac{\lambda\delta+s_{max}}{1-\theta}$ is the same as that under $r=\frac{\lambda\delta+s_{max}}{1-\theta}$. However, the server will incur more payments to clients under $r>\frac{\lambda\delta+s_{max}}{1-\theta}$, compared with the case when $r=\frac{\lambda\delta+s_{max}}{1-\theta}$. Hence, there is no need to discuss the case when the $r\in (\frac{\lambda\delta+s_{max}}{1-\theta},\infty)$.}:
	Since $r\geq \zeta_{3}$, all \textit{type-1} clients will participate in FL with AIGC-enhanced datasets (i.e., strategy $\varphi_{k}(1,1)$). While each \textit{type-2} client $k$ participates in FL with the AIGC-enhanced dataset when $r\geq \zeta_{2}(k)\Rightarrow \lambda_{k}\leq -\frac{\lambda}{\lambda\delta+\theta r}s_{k}+\frac{r \lambda}{\lambda\delta+\theta r}$. Combining two types of clients, we summarize that client whose attribute tuple $(s_{k}, \lambda_{k})$ is located in the green region of Fig. \ref{figincomplte}(d) will participate in FL with AIGC-enhanced datasets. Consequently, the probability that client $k$ participates in FL can be written as $P_3=\iint\limits_{D_{3}} u(s_{k})v(\lambda_{k}){\rm d}s_{k}{\rm d}\lambda_{k}$
	%\begin{eqnarray}
	%P_3=\iint\limits_{D_{3}}  u(s_{k})v(\lambda_{k}){\rm d}s_{k}{\rm d}\lambda_{k},
	%\end{eqnarray}
	where $D_{3}=D_{31}\cup D_{32}$, $D_{31}=\{(s_{k},\lambda_{k})\vert s_{k}\in(0, (1-\theta)r-\lambda\delta), \lambda_{k}\in (0,\lambda)\}$ and $D_{32}=\{(s_{k},\lambda_{k})\vert s_{k}\in [(1-\theta)r-\lambda\delta,s_{max}), \lambda_{k}\in (0,-\frac{\lambda}{\lambda\delta+\theta r}s_{k}+\frac{r \lambda}{\lambda\delta+\theta r}]\}$.	
\end{itemize}

%Obviously, lower values for $\lambda_{k}$ and $s_{k}$ indicate higher quality of clients from the server's perspective. Fig. \ref{figincomplte}(b)-(d) shows that our incentive mechanism strikes a balance between data quality and data-computing cost when recruiting clients.

%the lower $\lambda_{k}$ and $s_{k}$ for a client, the higher quality of the client. The shape of regions (green and blue regions in Fig. \ref{figincomplte}(b)-(d) depend on the specific incentive strategy of the server. Fig. \ref{figincomplte}(b)-(d) imply that our incentive strategy strikes a balance between data quality and data -computing cost when recruiting clients.

\subsection{The Server's Expected Cost}
Due to the inherent uncertainties in incomplete information scenario, the server endeavors to minimize its expected cost by finding the optimal strategy:
%In the context of an incomplete information scenario, the server endeavors to minimize its expected cost due to inherent uncertainties:
\begin{align}
&\mathop{Minimize}\limits_{T, r} \  \mathbb{E}(\mathcal{C}_{cost}),\label{Scost_problem_in}
\\ &\mbox{s.t.}\quad T\in \{1,2,...\}, \tag{\ref{Scost_problem_in}{a}} \label{Scost_problem_ina}
\\ &  \quad \quad \,  r \in (0,+\infty) \tag{\ref{Scost_problem_in}{b}} \label{Scost_problemb_ina}.
\end{align}
It is challenging to solve problem \eqref{Scost_problem_in} due to the following two aspects. On the one hand,
the calculation for $\mathbb{E}(\mathcal{C}_{cost})$ is non-trivial under a given strategy $(T,r)$ due to the complex clients' behaviors. On the other hand, the mutual coupling of $T$ and $r$ exacerbates the difficulty to derive the optimal strategy for the server in presence of information asymmetry. 

To estimate $\mathbb{E}(\mathcal{C}_{cost})$ under a given strategy $(T,r)$, we first rewrite \eqref{Scost_problem_in} as follows:
%We first introduce how to estimate $\mathbb{E}(\mathcal{C}_{cost})$ under a given strategy $(T,r)$ by rewritten \eqref{Scost_problem_in} as follows:
\begin{eqnarray}\label{scost_in}
\mathbb{E}(\mathcal{C}_{cost})=\gamma_{1}\mathbb{E}(M_{loss})+ \gamma_{2}\mathbb{E}(R_{total}),
\end{eqnarray}
where
\begin{small}
	\begin{eqnarray}
	\notag
	\!\!\!& &\!\!\! \mathbb{E}(M_{loss})=\phi^{hT}\Theta + (1-\phi^{hT})\kappa_{1}\frac{\psi}{\sqrt{h}}\underbrace{\mathbb{E}\left[\sum_{k\in\mathcal{N}_r}\frac{\sqrt{d_{k}}}{d(\mathcal{N}_r)}\right]}_{\boldsymbol{M}_{1}}\\
	\!\!\!& &\!\!\!+
	\underbrace{\mathbb{E}\left[\sum_{k\in\mathcal{N}^{L}_r}\frac{d_{k}}{d(\mathcal{N}_r)}\lambda_{k}+\sum_{k\in\mathcal{N}^{A}_r}\frac{d_{k}}{d(\mathcal{N}_r)}\lambda_{k,A}\right]}_{\boldsymbol{M}_{2}},
	\end{eqnarray}
\end{small}
and 
\begin{small}
	\begin{eqnarray}
	\mathbb{E}(R_{total})=\gamma_{2}Tr \underbrace{\mathbb{E}(\sum_{k\in \mathcal{N}_{r}}d_{k})}_{\boldsymbol{M}_{3}}(1-\frac{\mathbb{E}(\lambda_{k})}{\lambda}).
	\end{eqnarray}
\end{small}
To calculate \eqref{scost_in}, we next introduce how to estimate $\boldsymbol{M}_1$, $\boldsymbol{M}_{2}$ and $\boldsymbol{M}_{3}$ for cases \textbf{(\textit{i})}, \textbf{(\textit{ii})} and \textbf{(\textit{iii})}:

\subsubsection{\textbf{Estimation of $\boldsymbol{M}_{1}$}}
The challenge for estimating $\boldsymbol{M}_{1}$ is twofold: First, when $d(\mathcal{N}_r)=\sum_{k\in \mathcal{N}_r}d_{k}=0$ (e.g., no client participates in FL), the term $\sum_{k\in\mathcal{N}_r}\frac{\sqrt{d_{k}}}{d(\mathcal{N}_r)}$ is undefined.
Second, the value of $d(\mathcal{N}_r)=\sum_{k\in \mathcal{N}_r}d_{k}$ is determined by the set of the recruited client $\mathcal{N}_r$.

In light of the first challenge, we utilize a large positive constant $\Omega$ to characterize the cost of gradient error when there is no client participating in FL (i.e., $d(\mathcal{N}_r)=0$) \footnote{According to assumption $\Vert \nabla F_{k}(\boldsymbol{w};\boldsymbol{\xi}_{k})-\nabla F_{k}(\boldsymbol{w})\Vert^{2}\leq \frac{\psi^{2}}{B_{k}}$, the term $\sum_{k\in\mathcal{N}_r}\frac{\sqrt{d_{k}}}{d(\mathcal{N})}$ characterizes the gradient error from stochastic sampling. When there is no data for federated training, we consider that the value of this term tends to $\infty$.}. Based on this, we construct a random variable $Q$ as follows:
\begin{small}
	\begin{eqnarray}\label{Q_ran}
	Q=\left\{
	\begin{aligned}
	&\sum_{k\in\mathcal{N}_r}\frac{\sqrt{d_{k}}}{d(\mathcal{N}_r)}, \qquad \mbox{if} \  d(\mathcal{N}_{r})\neq 0,\\
	&\Omega, \qquad    \mbox{if} \ d(\mathcal{N}_{r})= 0.\\
	\end{aligned}
	\right.
	\end{eqnarray}
\end{small}
By supplementing the definition of $\sum_{k\in\mathcal{N}_r}\frac{\sqrt{d_{k}}}{d(\mathcal{N})}$, the value of $\mathbb{E}(Q)$ can be regarded as the estimated value of $\boldsymbol{M}_{1}$. Given unit data reward benchmark $r$, directly calculating $\mathbb{E}(Q)$ necessitates an algorithm with complexity of $O(2^{K})$, which inevitably incurs a significant expense.

\begin{figure}[t]
	\centering
	\includegraphics[width=0.85\linewidth]{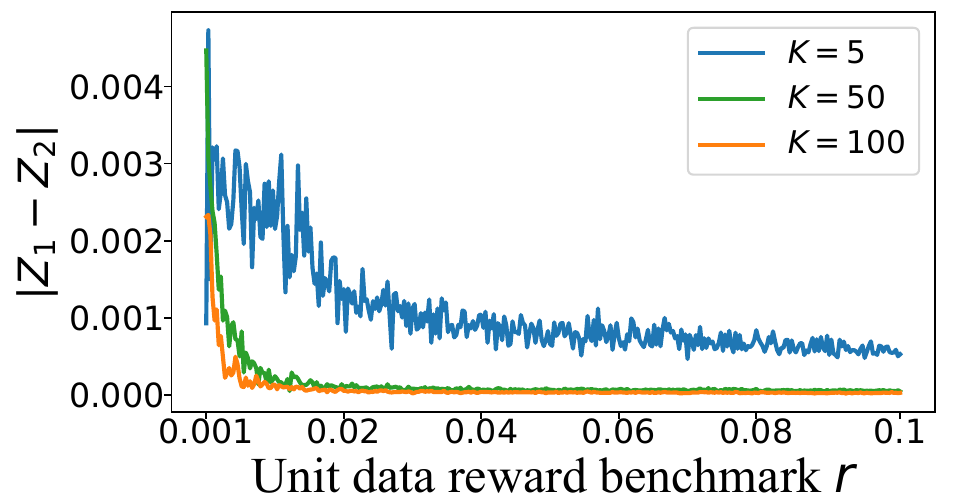}
	\caption{The value of $\vert Z_{1}-Z_{2} \vert$ for different unit data reward benchmark $r$.}
	\label{zgaps}
\end{figure}

To calculate $\mathbb{E}(Q)$ with high computational efficiency, we define a random variable $z_{k} \in \{0,1\}$ for each client $k \in \mathcal{K}$ to denote whether to participate in FL given the unit data reward benchmark $r$. Here, we let $z_{k}=1$ with probability $p$ and $z_{k}=0$ with probability $1-p$, where $p$ equals to $P_{1}$, $P_{2}$ and $P_{3}$ for cases \textbf{(\textit{i})}, \textbf{(\textit{ii})} and \textbf{(\textit{iii})}, respectively. Then, an estimator for random variable $Q$ is constructed by:
%For computational efficiency, we construct a new random variable to represent $Q$. Specifically, we define a random variable $z_{k}, k\in \mathcal{K}$ for each client $k$ to denote whether it will participate in FL given unit data reward $r$. Consequently, we have $x_{k}=1$ with probability $p$ and $x_{k}=0$ with probability $1-p$. Here, we have $p=P_{1},P_{2},P_{3}$ for cases \textbf{(\textit{i})}, \textbf{(\textit{ii})} and \textbf{(\textit{iii})}), respectively. Then, we construct a new random variable:
\begin{eqnarray}
\notag
Z=\frac{\sum_{k\in \mathcal{K}}z_{k}\sqrt{d_{k}}+\epsilon\cdot \prod_{k \in \mathcal{K}}(1-z_{k})}{\sum_{k\in \mathcal{K}}z_{k}d_{k}+\sqrt{\epsilon}\cdot \prod_{k \in \mathcal{K}}(1-z_{k})}+\Omega\cdot \prod_{k \in \mathcal{K}}(1-z_{k}),\!\!\!\!\!\!\!\!\!\!\!\!\\
\end{eqnarray}
where $\epsilon \to 0^+$. When no client participates in FL (i.e., $\mathcal{N}_r=\emptyset$), we have $Z=\sqrt{\epsilon}+\Omega\approx\Omega$. Otherwise, we have $Z=\sum_{k\in\mathcal{N}_r}\frac{\sqrt{d_{k}}}{d(\mathcal{N}_{r})}$. 
Then, we calculate $\mathbb{E}(Z)$ as follows: 
%\begin{small}
\begin{eqnarray}\label{Y_random}
\notag
\mathbb{E}(Z)\!\!\!\!\!\!&=&\!\!\!\!\!\!\underbrace{\mathbb{E}\left(\frac{\sum_{k\in \mathcal{K}}z_{k}\sqrt{d_{k}}+\epsilon\cdot \prod_{k \in \mathcal{K}}(1-z_{k})}{\sum_{k\in \mathcal{K}}z_{k}d_{k}+\sqrt{\epsilon}\cdot\prod_{k \in \mathcal{K}}(1-z_{k})}\right)}_{Z_1}+(1-p)^{K} \Omega \!\!\!\\
\notag
\!\!\!\!\!\!&\approx&\!\!\!\!\!\!\underbrace{\frac{\mathbb{E}(\sum_{k\in \mathcal{K}}z_{k}\sqrt{d_{k}}+\epsilon \cdot \prod_{k \in \mathcal{K}}(1-z_{k}))}{\mathbb{E}(\sum_{k\in \mathcal{K}}z_{k}d_{k}+\sqrt{\epsilon}\cdot \prod_{k \in \mathcal{K}}(1-z_{k}))}}_{Z_2}+(1-p)^{K} \Omega\\
\notag
\!\!\!\!\!\!&=&\!\!\!\!\!\! \frac{\sum_{k\in \mathcal{K}}p \sqrt{d_{k}}+\epsilon\cdot (1-p)^{K}}{\sum_{k\in \mathcal{K}}pd_{k}+\sqrt{\epsilon}\cdot(1-p)^{K}}+(1-p)^{K} \Omega\\
\!\!\!\!\!\!&\approx&\!\!\!\!\!\! \frac{\sum_{k\in \mathcal{K}} \sqrt{d_{k}}}{\sum_{k\in \mathcal{K}}d_{k}}+(1-p)^{K} \Omega.
\end{eqnarray}
%\end{small}
In the second equation of \eqref{Y_random}, we use term $Z_{2}$ to approximate term $Z_{1}$ when $\mathcal{N}_r \neq \emptyset$. As illustrated in Fig. \ref{zgaps}, we show the difference between the real value of $Z_{1}$ and $Z_{2}$ under different values of unit data reward benchmark $r$. For parameter settings, we set $\lambda_{k}\sim Uniform(0,3)$, $s_{k}\sim Uniform(0,0.1)$ and $\epsilon=10^{-8}$. The other parameters' settings are the same as that in Section \ref{dis_section}. We can see that $\vert Z_{1}- Z_{2}\vert$ is small enough with different numbers of candidate clients $K$, indicating that our approximation is reasonable.

By formulating a two-tier nested estimation framework,
the quantification of $\mathbb{E}(Q)$ is attained through the estimation of $\mathbb{E}({Z})$, leading to the ultimate derivation of the value of $\boldsymbol{M}_1$.

%the real value of $Z_{1}$ (by averaging the results over 200 times) and its corresponding value of $Z_{2}$ when $\mathcal{N}_{r}\neq \emptyset$ based on the unit data reward $r$ in different number of candidate client. For parameter settings, we set $\lambda_{k}\sim Uniform(0,3)$, $s_{k}\sim Uniform(0,0.1)$ and $\epsilon=10^{-8}$. The other parameters' settings are same as that in section \ref{dis_section}. We can see that our approximation is available, since the  $\vert Z_{1}- Z_{2}\vert$ is small enough.

\subsubsection{\textbf{Calculation of $\boldsymbol{M}_{2}$}} 
As $\boldsymbol{M}_{2}$ indicates the expected training gradient error caused by data quality, we have $\sum_{k\in\mathcal{N}^{L}}\frac{d_{k}}{d(\mathcal{N}_r)}\lambda_{k}+\sum_{k\in\mathcal{N}^{A}}\frac{d_{k}}{d(\mathcal{N}_{r})}\lambda_{k,A}=0$ when $d(\mathcal{N}_r)=0$ (i.e., $\mathcal{N}_{r}= \emptyset$). Recalling cases \textbf{(\textit{i})}, \textbf{(\textit{ii})} and \textbf{(\textit{iii})} discussed in Section \ref{subsec_client}, we have $\mathcal{N}_{r}^{L}= \emptyset$ or $\mathcal{N}_{r}^{A}=\emptyset$ when $\mathcal{N}_{r}=(\mathcal{N}_{r}^{L}\cup \mathcal{N}_{r}^{A}) \neq \emptyset$. 
Hence, we have $\boldsymbol{M}_{2}=\mathbb{E}\left[\sum_{k\in\mathcal{N}_{r}}\frac{d_{k}}{d(\mathcal{N}_{r})}\lambda_{k}\right]$ for cases \textbf{(\textit{i})} and \textbf{(\textit{ii})}. In case \textbf{(\textit{iii})}, we have $\boldsymbol{M}_{2}=\mathbb{E}\left[\sum_{k\in\mathcal{N}_{r}}\frac{d_{k}}{d(\mathcal{N})}\lambda_{k,A}\right]$.

For case \textbf{(\textit{i})}, we can obtain the value of $\boldsymbol{M}_{2}$ by calculating the following equation:
%For case \textbf{(\textit{i})}, we have:
\begin{eqnarray}\label{lamd_rand1}
\mathbb{E}\left[\sum_{k\in\mathcal{N}_{r}}\frac{d_{k}}{d(\mathcal{N}_{r})}\lambda_{k}\right]\!\!=\!\!\mathbb{E}(\lambda_{k})\!\!=\iint\limits_{D_{1}} u(s_{k})v(\lambda_{k})\lambda_{k}{\rm d}s_{k}{\rm d}\lambda_{k}.
\end{eqnarray}
Here, the first equation originates from $\sum_{k\in\mathcal{N}_{r}}\frac{d_{k}}{d(\mathcal{N})}=1$ when $\mathcal{N}_{r}\neq 0$. Similarly, we replace the region $D_{1}$ with $D_{2}$ in \eqref{lamd_rand1} and obtain the value of $\boldsymbol{M}_2$ for case \textbf{(\textit{ii})}. As each recruited clients use AIGC-enhanced datasets for FL model training, the value of $\boldsymbol{M}_2$ for case \textbf{(\textit{iii})} can be calculated by:
%For case \textbf{(\textit{iii})}, each recruited clients will use generated datasets in FL process. From $\lambda_{k,A}=\theta \lambda_{k}$, we have:
\begin{eqnarray}\label{lamd_rand3}
\notag
\!\!\mathbb{E}\!\!\left[\sum_{k\in\mathcal{N}_{r}}\frac{d_{k}}{d(\mathcal{N})}\lambda_{k,A}\right]\!\!=\!\!\mathbb{E}(\lambda_{k,A})=\!\!\iint\limits_{D_{3}} u(s_{k})v(\lambda_{k})\theta \lambda_{k}{\rm d}s_{k}{\rm d}\lambda_{k}.\!\!\!\!\!\!\!\!\!\!\!\!\!\!
\\
\end{eqnarray}
Here, the second equation originates from $\lambda_{k,A}=\theta \lambda_{k}$.

\subsubsection{\textbf{Calculation of $\boldsymbol{M}_{3}$}} We have $\boldsymbol{M}_{3}=\mathbb{E}(\sum_{k\in \mathcal{K}}z_{k}d_{k})=p(\sum_{k\in \mathcal{K}}d_{k})$.

By substituting $\boldsymbol{M}_1$, $\boldsymbol{M}_2$ and $\boldsymbol{M}_3$ into \eqref{scost_in}, we can obtain the server's expected cost given fixed $T$ and $r$.

\subsection{Server's Optimal Strategy in Incomplete Information Scenario}
%Based on the above analysis, we aim obtain the server's strategy in incomplete information scenario.

In light of the mutual coupling of $T$ and $r$ in server's strategy, we first aim to obtain the optimal $r$ by minimizing \eqref{Scost_problem_in} given a fixed $T$. Due to different clients' behaviors over different ranges of $r$, problem \eqref{Scost_problem_in} is naturally decomposed into three subproblems corresponding to cases \textbf{(\textit{i})}, \textbf{(\textit{ii})} and \textbf{(\textit{iii})}, denoted as $\mathbb{E}[\mathcal{C}_{cost, \it {(i)}}(r)], r\in(0, s_{max})$; $\mathbb{E}[\mathcal{C}_{cost, \it {(ii)}}(r)], r\in [s_{max},\zeta_{3})$ and $\mathbb{E}[\mathcal{C}_{cost, \it {(iii)}}(r)],r\in [\zeta_{3}, \frac{\lambda\delta+s_{max}}{1-\theta}]$, respectively.
Taking case \textbf{(\textit{i})} as an example, the server's cost under a fixed $T$ can be represented as:
\begin{eqnarray}\label{incom_r}
\notag
\!\!\!\!\!\!\!\!\!\!\!\!\!\!\!& &\!\!\!\mathbb{E}(\mathcal{C}_{cost,(i)})(r)=\gamma_{1}\phi^{hT}\Theta +\gamma_{1} (1-\phi^{hT})\kappa_{1}\frac{\psi}{\sqrt{h}}\mathbb{E}(Z)\\
\!\!\!\!\!\!\!\!\!\!\!\!\!\!\!& &\!\!\!+ \gamma_{1} (1-\phi^{hT})\kappa_{1}
\mathbb{E}(\lambda_{k})+\gamma_{2}Tr\mathbb{E}(\sum_{k\in \mathcal{N}_{r}}d_{k})(1-\frac{\mathbb{E}(\lambda_{k})}{\lambda}).
\end{eqnarray}
%Given the fixed number of global iteration $T$, for each case (\textbf{(\textit{i})},\textbf{(\textit{ii})} and \textbf{(\textit{iii})}), we can substitute $\mathbb{E}(Z)$, $\mathbb{E}(\lambda_{k})$ and $\mathbb{E}(\sum_{k\in \mathcal{N}_{r}}d_{k})$ into the following equation:
These three subproblems are all single-objective optimization problems in terms of $r$, each involving double integrals with respect to $s_{k}$ and $\lambda_{k}$, and can be easily solved by traditional optimizers. The solution corresponding to the minimum server's cost among these three subproblems signifies the optimal unit data reward benchmark of $r$ for a given fixed $T$.

\noindent \textbf{Remark} 5: 
A low unit data reward benchmark $r$ may lead to a large $\mathbb{E}(Z)$ in \eqref{incom_r}, which indicates that the server may suffer from high cost to hedge the risk of no or very few clients participating in FL model training. With a large unit data reward benchmark $r<\zeta_{3}$, there are more clients with low data quality participating in FL, resulting in degraded model performance and high payments. When $r\geq \zeta_{3}$, the server recruits clients to participate in FL with AIGC-enhanced datasets, leading to satisfactory training performance but high payment to clients. Hence, a proper decision of $r$ is significantly important to realize a nice balance between model performance and payment to the clients.

\begin{algorithm}[t]
	\small
	\caption{Server's optimal strategy in incomplete information scenario}
	\label{incompleteagl}
	\begin{algorithmic}[1]
		\Require The candidate client set $\mathcal{K}$.
		\Ensure Server's optimal strategy profile $(T_{o}, r_{o})$.
		\State Obtain $T_{max}$ by solving the problem \eqref{Scost_problem_tt}.
		\State $\mathcal{S}\gets \emptyset$.
		\For {$T=1$ to $T=T_{max}$}
		\State Calculate the minimum server's cost among the three subproblems $\mathbb{E}[\mathcal{C}_{cost, \it {(i)}}(r)]$; $\mathbb{E}[\mathcal{C}_{cost, \it {(ii)}}(r)]$ and $\mathbb{E}[\mathcal{C}_{cost, \it {(iii)}}(r)]$ and obtain the corresponding reward as the optimal unit data reward benchmark $r_T$.
		%\State Obtain the optimal unit data reward $r_{T}$ given the number global iteration $T$ by comparing the solution corresponding to the minimum server's cost among these three subproblems $\mathbb{E}[\mathcal{C}_{cost, \it {(i)}}(r)]$; $\mathbb{E}[\mathcal{C}_{cost, \it {(ii)}}(r)]$ and $\mathbb{E}[\mathcal{C}_{cost, \it {(iii)}}(r)]$.
		\State $\mathcal{S}\gets \mathcal{S}\cup \{(T,r_{T})\}$.
		\EndFor
		\State Search the optimal strategy profile in $\mathcal{S}$ such that minimizes server's cost, which is denoted as $(T_{o},r_{o})$.
		\State \textbf{return} strategy profile $(T_{o},r_{o})$.
	\end{algorithmic}
\end{algorithm}

We have determined the optimal $r$ for any fixed $T$, which inspires us to globally select the optimal $T$, since $T$ belongs to a limited integer set $[0,T_{max}]$. 
Next, we introduce how to determine the value of $T_{max}$.
Proposition \ref{pro3} reveals the optimal $T^{\ast}(r)$ given a unit data reward benchmark $r$. Then, the optimal number of global iterations $T_{o}$ must satisfy $T_{o}\leq T_{max}$, where $T_{max}=\max_{r}T^{\ast}(r),\forall r \in (0,\frac{\lambda\delta+s_{max}}{1-\theta}]$. We summarize this conclusion as follows: 
\begin{prop}\label{pro4}
	\textit{In incomplete information scenario, $T_{max}$ can be determined by solving the following optimization problem $T_{max}=\mathop{Maximize}\limits_{r} \  T^{\ast}$, i.e.,
		\begin{align}
		&\mathop{Maximize}\limits_{r} \log_{\phi^h}\left[\frac{\gamma_{2}\mathbb{E}(R(\mathcal{N}_r))}{\gamma_{1}(-\ln \phi^h)(\Theta-\kappa_{1}\mathbb{E}(\Lambda(\mathcal{N}_r)))}\right],\label{Scost_problem_tt}
		\\ & \qquad \qquad \qquad \mbox{s.t.}\quad r\in (0,\frac{\lambda\delta+s_{max}}{1-\theta}], \tag{\ref{Scost_problem_tt}{a}} \label{Scost_problem_tta}
		%\\ &  \quad \quad \,  \mathbb{E}(R(\mathcal{N}_{r}))=r\mathbb{E}(\sum_{k\in \mathcal{N}_{r}}d_{k})(1-\frac{\mathbb{E}(\lambda_{k})}{\lambda}) \tag{\ref{Scost_problem_tt}{b}} \label{Scost_problem_ttb}.
		%\\ &  \quad \quad \, \mathbb{E}(\Lambda(\mathcal{N}_{r}))=\gamma_{1}\kappa_{1}\frac{\psi}{\sqrt{h}}\mathbb{E}(Z)+\gamma_{1}\kappa_{1}
		%\mathbb{E}(\lambda_{k}) \tag{\ref{Scost_problem_tt}{c}} \label{Scost_problem_ttc}.
		\end{align}
		The optimization problem can be decomposed to three subproblems, since $\mathbb{E}(R(\mathcal{N}_{r}))$ and $\mathbb{E}(\Lambda(\mathcal{N}_{r}))$ are different in cases \textbf{(\textit{i})}, \textbf{(\textit{ii})} and \textbf{(\textit{iii})}. Specifically, we have $\mathbb{E}(R(\mathcal{N}_{r}))=r\mathbb{E}(\sum_{k\in \mathcal{N}_{r}}d_{k})(1-\frac{\mathbb{E}(\lambda_{k})}{\lambda})$ and $\mathbb{E}(\Lambda(\mathcal{N}_{r}))=\gamma_{1}\kappa_{1}\frac{\psi}{\sqrt{h}}\mathbb{E}(Z)+\gamma_{1}\kappa_{1}
		\mathbb{E}(\lambda_{k})$ for cases \textbf{(\textit{i})} and \textbf{(\textit{ii})}. For case \textbf{(\textit{iii})}, we have $\mathbb{E}(R(\mathcal{N}_{r}))=r\mathbb{E}(\sum_{k\in \mathcal{N}_{r}}d_{k})(1-\frac{\mathbb{E}(\lambda_{k,A})}{\lambda})$ and $\mathbb{E}(\Lambda(\mathcal{N}_{r}))=\gamma_{1}\kappa_{1}\frac{\psi}{\sqrt{h}}\mathbb{E}(Z)+\gamma_{1}\kappa_{1}
		\mathbb{E}(\lambda_{k,A})$.
	}
\end{prop}
The optimization problem \eqref{Scost_problem_tt} is a single-objective which can be solved by traditional optimizer. Based on the above dicussion, Algorithm 2 summarizes the procedure of obtaining server's optimal strategy in incomplete information scenario.

\section{Performance Evaluation}\label{PerE}
In this section, we conduct experiments to study the impact of  different distributions of parameters $s_{k}$ and $\lambda_{k}$ on server's strategy and cost. Further, we evaluate the performance of our mechanism in incomplete information scenario, compared with complete information scenario. Finally, we compare the training performance of our mechanism with two benchmarks mechanism on real datasets.
\subsection{Parameter Settings}
\textbf{Experimental environment.} We conduct experiments on the device equipped with Ubuntu 18.04.05, CUDA v12.0, GPU (Tesla P100-PCIE) and Intel(R) Xeon(R) CPU (E5-2678 v3).

\textbf{Datasets and models.} To gauge the effectiveness of our incentive mechanism, we consider image classification as the FL training task and conduct extensive evaluations with three widely-used real-world datasets for FL: MNIST \cite{lecun1998gradient}, CIFAR10 \cite{krizhevsky2009learning} and GTSRB \cite{gtsrb13}. We employ a multi-layer perception network consisting of a single hidden layer with 256 hidden units as the learning model for MNIST dataset. While we use LeNet \cite{lecun1998gradient}, which consists of two sets of convolution and pooling layers, then two fully-connected layers with ReLU activation, as the model trained on clients for the complex CIFAR10 dataset. The architecture of model for GTSRB dataset comprises two convolutional layers followed by max-pooling operations, culminating in two fully connected layers.

\textbf{Data synthesis by AIGC service.} Vision models such as diffusion model have demonstrated impressive capability in high-quality image synthesis \cite{ho2022classifier}. In this paper, we leverage a pre-trained diffusion model in \cite{ho2020denoising} to provision the AIGC service for the data synthesis of MNIST, CIFAR10 and GTSRB datasets.

\textbf{Parameter settings.} The values of $\rho$, $\mu$, and $\lambda_{k}$ are dictated by the specific loss function and dataset, which can be estimated within a concise FL training process empirically \cite{Adaptive19}. For the parameter estimation of MNIST dataset, we set $K = 10$ and $\psi=25$. Each client is randomly allocated with $5,000$ data samples under a uniform data distribution over 10 classes. With a learning rate of $\eta=0.01$ and global iteration number of $T=50$, we simulate the FL training process and obtain the estimated value for each parameter: $\rho=37.36$, $\mu=5.48$, and $\beta=0.57$. Similarly, we can obtain  $\rho=16.94$, $\mu=2.53$, and $\beta=0.28$ for CIFAR10 dataset by setting $K=10$ and $\psi=10$. Also, we have $\rho=12.49$, $\mu=3.40$, and $\beta=0.78$ for GTSRB dataset by setting $K=43$ and $\psi=10$.

\begin{table}[h]
	\setlength{\abovecaptionskip}{-0.02cm}
	\renewcommand{\arraystretch}{1.3}
	\caption{The range of $\lambda_{k}$ with $l$ classes for MNIST dataset.}
	\label{table_para1}
	\centering
	\scriptsize
	\begin{tabular}{c|c|c|c}
		\hline
		$l$ & $\lambda_{k}$ &$l$  & $\lambda_{k}$ \\
		\hline
		1  &$(0.8,3)$& 8& $(0.14,0.18)$\\
		\hline
		2,3  &$(0.42,0.8)$& 9& $(0.08,0.14)$\\
		\hline
		4,5  &$(0.32,0.42)$& 10& $(0,0.08)$\\
		\hline
		6,7  &$(0.18,0.32)$& & \\
		\hline
	\end{tabular}
\end{table}
\begin{table}[h]
	\setlength{\abovecaptionskip}{-0.03cm}
	\renewcommand{\arraystretch}{1.3}
	\caption{The range of $\lambda_{k}$ with $l$ classes for CIFAR10 dataset.}
	\label{table_para2}
	\centering
	\scriptsize
	\begin{tabular}{c|c|c|c}
		\hline
		$l$  &  $\lambda_{k}$ & $l$  &  $\lambda_{k}$ \\
		\hline
		1,2  &$(0.76,2)$& 6,7& $(0.35,0.46)$\\
		\hline
		3  &$(0.65,0.76)$& 8& $(0.23,0.35)$\\
		\hline
		4  &$(0.55,0.65)$& 9& $(0.15,0.23)$\\
		\hline
		5  &$(0.46,0.55)$& 10& $(0,0.15)$\\
		\hline
	\end{tabular}
\end{table}

For $g_{data}$ and $g_{diff}$ on MNIST datasets, we conduct a simple FL training process (e.g., training within 20 global iterations) for both local dataset and generated dataset where all data samples are generated by AIGC service. In order to evaluate the maximum value of the gradient and gradient difference, we simulate the highly non-IID scenario for the two datasets and simultaneously record the maximum gradient norm, and maximum gradient error between local and generated datasets. As a result, we obtain $g_{data}=\max_{t,k} \Vert \nabla F_{k}(\boldsymbol{w}^{t}; \mathcal{D}_{MNIST})\Vert=2.45$ and $g_{diff}=\max_{t,k} \Vert \nabla F_{k}(\boldsymbol{w}^{t}; \mathcal{D}_{MNIST})- \nabla F_{k}(\boldsymbol{w}^{t};\mathcal{D}_{AIGC})\Vert=1.05$ for MNIST dataset. Similarly, $g_{data}= 1.75$ and $g_{diff}= 0.54$ can be acquired for CIFAR10 dataset. In terms of GTSRB datasets, we have $g_{data}= 16.51$ and $g_{diff}= 12.07$.

%based on a simulated highly non-IID scenario (e.g., each client is assigned with data samples from only one class) in order to evaluate the maximum value of the gradient and gradient difference. 

%for both local dataset and generated dataset where all data samples are generated by AIGC service, in order to evaluate the maximum value of the gradient and gradient difference.

As for the non-IID degree $\lambda_{k}$ which is dataset-specified, we estimate the rough range of $\lambda_{k}$ under different data partition cases. Specifically, for the data partition case where client $k$ possesses data samples from $l$ classes, we conduct a simple FL training process with $T = 20$ global iterations and record the maximum gradient error between the local gradient of client $k$ and the global gradient, i.e., $\lambda_{k}(l)=\max_{t}\Vert \nabla F_{k}(\boldsymbol{w}^{t})- \nabla F(\boldsymbol{w}^{t})\Vert$. Following this way, we are able to obtain the series values  $\lambda_{k}(l), k\in \{1,2,...,10\}$ and calculate the rough range of $\lambda_{k}$ for each data partition cases of MNIST and CIFAR10 datasets, as summarized in Table \ref{table_para1} and Table \ref{table_para2}.

%Table \ref{table_para1} and Table \ref{table_para2} conclude the range of $\lambda_{k}$ under different data partition cases (e.g., client $k$ is assigned with data samples from $l \in {1,2,...,10}$ classes) for MNIST and CIFAR10 datasets, respectively.

%Taking CIFAR10 as an example, Fig. \ref{cifar_diffusion} verifies that the quality of data samples generated by diffusion model is satisfactory. Besides, we randomly select 10 clients and show the distribution among classes of the local datasets in Fig. \ref{local_dis}. Here, the area of a circle represents the number of data samples for a target class. By conducting data synthesis with diffusion model, we replenish data samples for the minority classes (circles without black borders as depicted in Fig. \ref{arg_dis}) and obtain high-quality (e.g., IID) generated datasets.% for the clients.

\subsection{Discussion on AIGC-enhanced dataset}\label{PEdis}
In this paper, we employ the AIGC-enhanced dataset, which comprises a mixture of real-world and generated data samples for each client, i.e., retaining some local data samples while introducing generated ones. From a theoretical perspective, we have demonstrated that the data quality of the AIGC-enhanced dataset is superior to that of a dataset consisting of only generated samples (referred to as the AIGC-only dataset) in Section \ref{clients_att}. To further validate our conclusion, we have added experiments to evaluate the training performance of the AIGC-enhanced dataset on the MNIST and CIFAR10 datasets.

In our experiments, we set $K=10$ clients and $T=100$ global iterations for both MNIST and CIFAR10 dataset. Under identical training conditions, we compare training accuracy on original local datset, AIGC-enhanced datasets, and AIGC-only dataset, respectively. For instance, as shown in Fig. \ref{Mnist_200_1}, we randomly assign $200$ data samples from one class as local dataset for each client. Based on this, we construct AIGC-enhanced dataset (IID dataset) with $10$ classes for each client, comprising a mixture of real-world and generated data samples (discussed in Section II-A in revised paper). For comparison, we generate $200$ data samples with $10$ classes for each client as AIGC-only dataset (IID dataset). 

As illustrated in Fig. \ref{Mnist_200_1}, Fig. \ref{Mnist_200_5}, and Fig. \ref{Mnist_500_5}, the AIGC-enhanced dataset consistently achieves the highest accuracy across various settings. The poor performance of adoption AIGC-only dataset is attributed to the distribution differences between generated data samples (e.g., $\boldsymbol{p}_{a}$) and real-world data samples (e.g., $\boldsymbol{p}$).
On the MNIST dataset, when clients possess a small number of classes, the proportion of generated data samples is higher in the AIGC-enhanced dataset, resulting in similar performance compared with the AIGC-only dataset. 
Similarly, superior training accuracy of the AIGC-enhanced dataset on CIFAR10 datasets is evident in Fig. \ref{Cifar_500_3}, Fig. \ref{Cifar_500_5}, and Fig. \ref{Cifar_1000_5}.
Notably, the AIGC-enhanced dataset not only demonstrates excellent training accuracy but also significantly reduces the costs associated with the data generation process compared to the AIGC-only dataset. The results obtained from the MNIST and CIFAR10 datasets corroborate our theoretical findings and highlight the advantages of utilizing the AIGC-enhanced dataset.

\begin{figure*}[t]
	\begin{minipage}[t]{0.33\textwidth}
		\setlength{\abovecaptionskip}{-0.05cm}
		\centering
		\includegraphics[width=\linewidth]{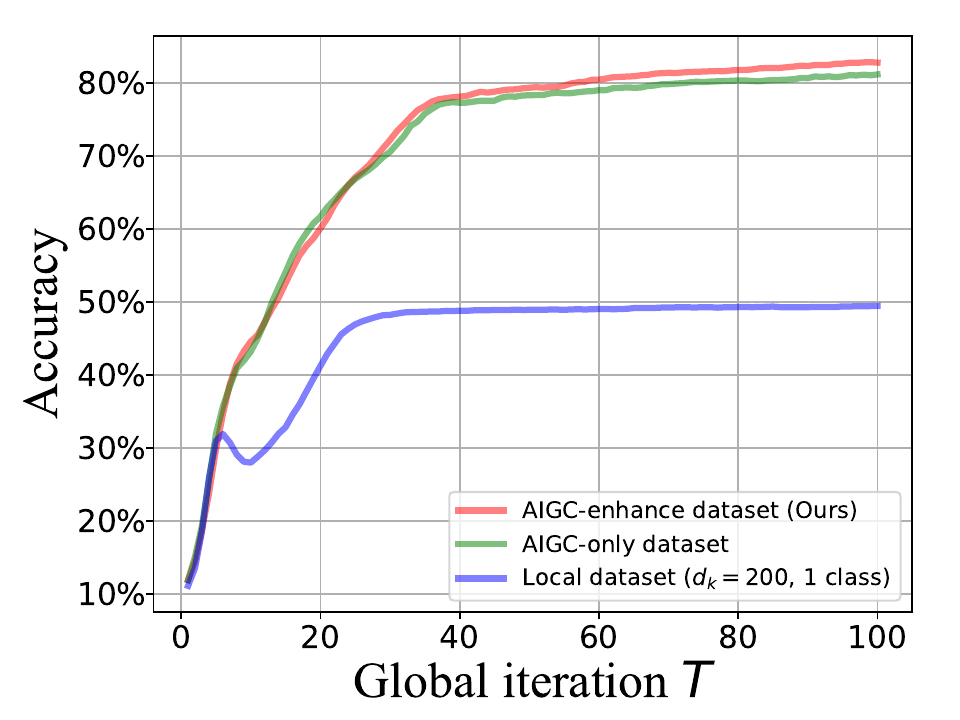}
		\caption{Accuracy of each generated datasets when local dataset with $d_{k}=200$ and $1$ class on MNIST dataset.}
		\label{Mnist_200_1}
	\end{minipage}
	\begin{minipage}[t]{0.33\textwidth}
		\setlength{\abovecaptionskip}{-0.05cm}
		\centering
		\includegraphics[width=\linewidth]{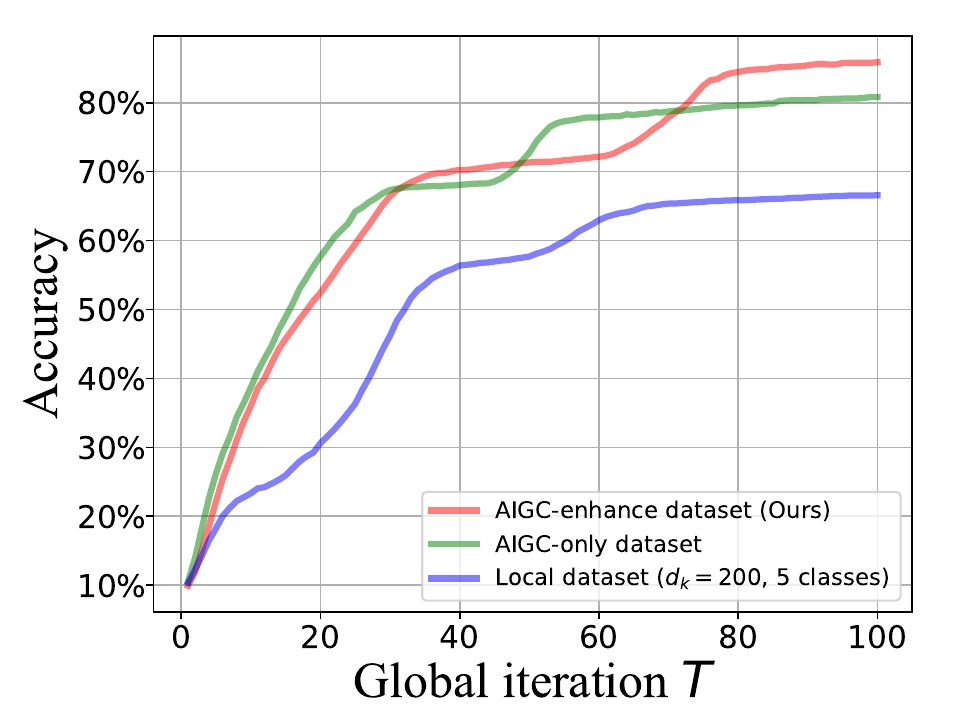}
		\caption{Accuracy of each generated datasets when local dataset with $d_{k}=200$ and $5$ classes on MNIST dataset.}
		\label{Mnist_200_5}
	\end{minipage}
	\begin{minipage}[t]{0.33\textwidth}
		\setlength{\abovecaptionskip}{-0.05cm}
		\centering
		\includegraphics[width=\linewidth]{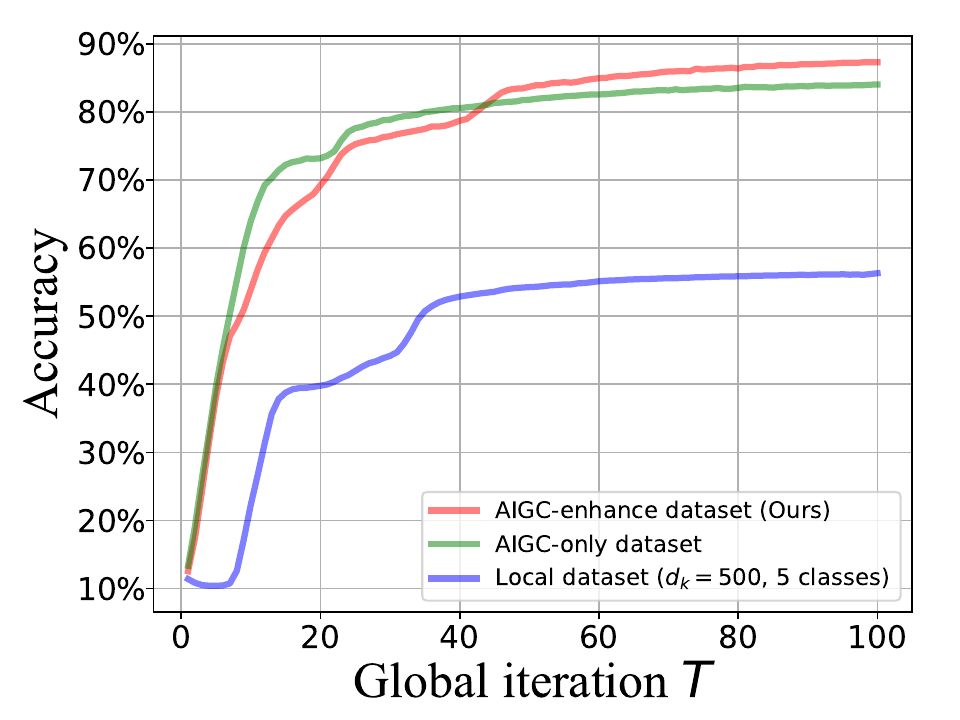}
		\caption{Accuracy of each generated datasets when local dataset with $d_{k}=500$ and $5$ classes on MNIST dataset.}
		\label{Mnist_500_5}
	\end{minipage}
	\vspace{-0.3cm}
\end{figure*}
\begin{figure*}[t]
	\begin{minipage}[t]{0.33\textwidth}
		\setlength{\abovecaptionskip}{-0.05cm}
		\centering
		\includegraphics[width=\linewidth]{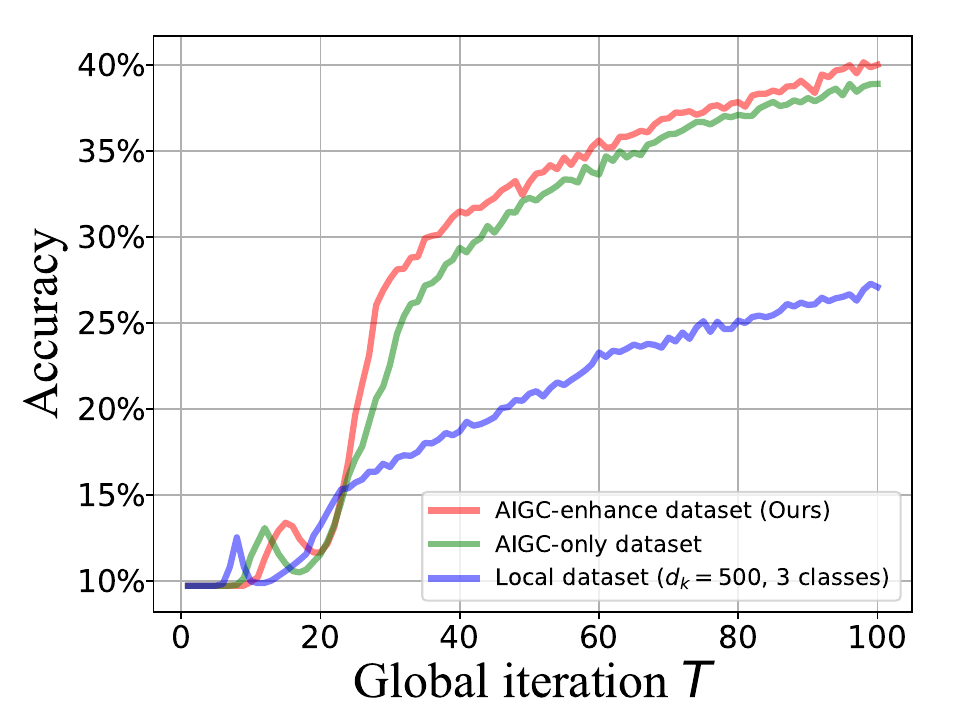}
		\caption{Accuracy of each generated dataset when local dataset with $d_{k}=500$ and $3$ classes on CIFAR10 dataset.}
		\label{Cifar_500_3}
	\end{minipage}
	\begin{minipage}[t]{0.33\textwidth}
		\setlength{\abovecaptionskip}{-0.05cm}
		\centering
		\includegraphics[width=\linewidth]{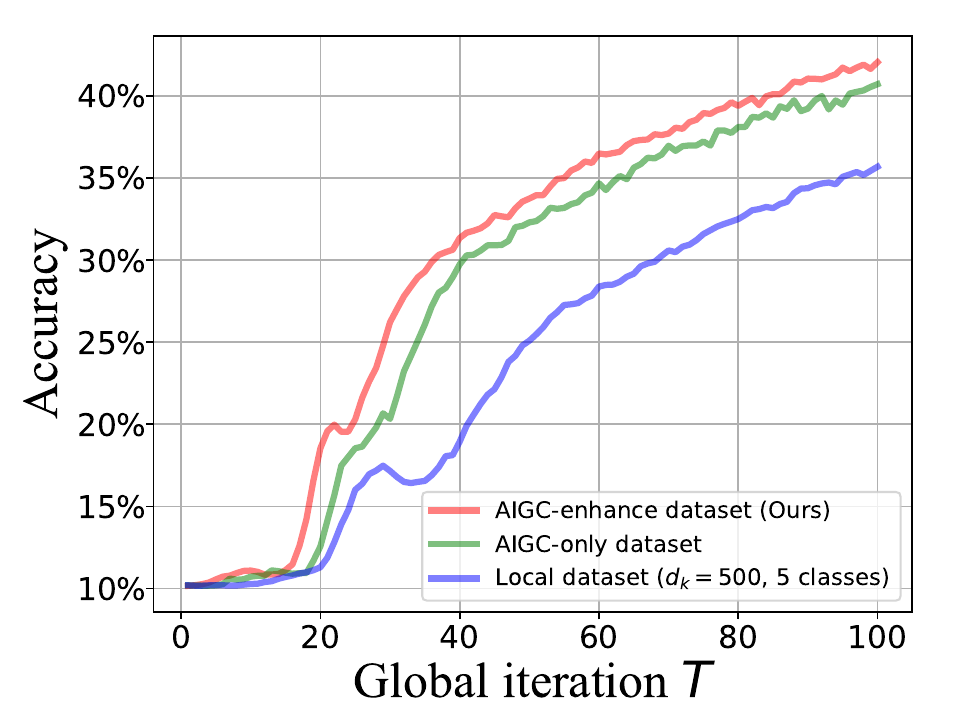}
		\caption{Accuracy of each generated dataset when local dataset with $d_{k}=500$ and $5$ classes on CIFAR10 dataset.}
		\label{Cifar_500_5}
	\end{minipage}
	\begin{minipage}[t]{0.33\textwidth}
		\setlength{\abovecaptionskip}{-0.05cm}
		\centering
		\includegraphics[width=\linewidth]{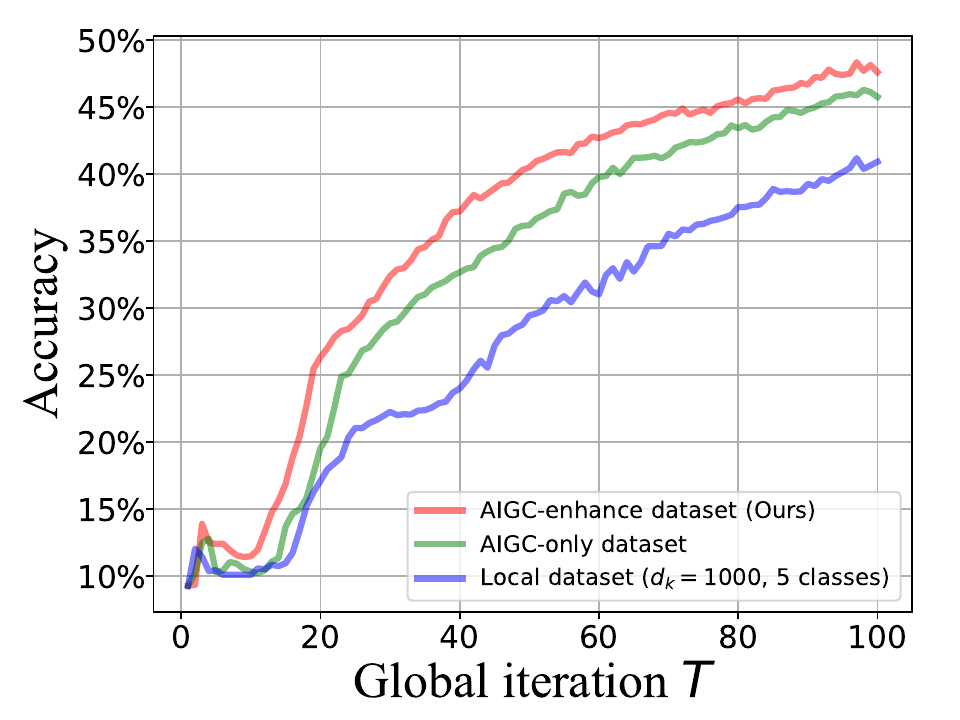}
		\caption{Accuracy of each generated dataset when local dataset with $d_{k}=1000$ and $5$ classes on CIFAR10 dataset.}
		\label{Cifar_1000_5}
	\end{minipage}
	\vspace{-0.3cm}
\end{figure*}

\begin{figure}[h]
	\centering
	\includegraphics[height=0.26\textwidth]{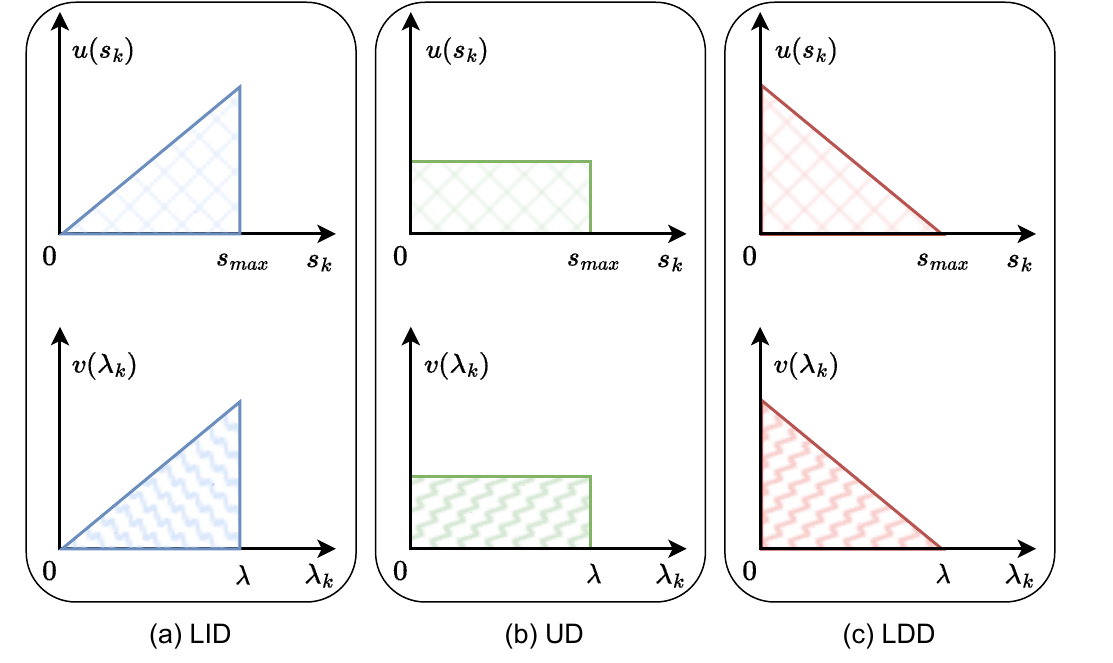}
	\caption{An illustration of three different kinds of distributions for $s_k$ and $\lambda_k$.}
	\label{fig_three_distribution}
\end{figure}

\subsection{Impact of Distributions of $s_{k}$ and $\lambda_{k}$ on Server's Strategy}\label{dis_section}
For the ease of investigating the impact of the distribution of client's attributes (e.g., $s_{k}$ and $\lambda_{k}$) on server's strategy, we consider the following three simple distributions for $s_{k}\in (0,s_{max}]$ and $\lambda_{k}\in (0,\lambda)$: 
\begin{itemize}
	\item Linear increasing distribution (LID): probability density functions of $s_{k}$ and $\lambda_{k}$ are set to $u(s_{k})=\frac{2}{s_{max}^{2}}s_{k}$ and $v(\lambda_{k})=\frac{2}{\lambda^{2}}\lambda_{k}$, respectively.
	\item Uniform distribution (UD): $u(s_{k})=\frac{1}{s_{max}}$ and $v(\lambda_{k})=\frac{1}{\lambda}$ are adopted for $s_{k}$ and $\lambda_{k}$, respectively.
	\item Linear decreasing distribution (LDD): probability density functions of $u(s_{k})=-\frac{2}{s_{max}^{2}}s_{k}+\frac{2}{s_{max}}$ and $v(\lambda_{k})=-\frac{2}{\lambda^{2}}\lambda_{k}+\frac{2}{\lambda}$ are utilized for $s_{k}$ and $\lambda_{k}$, respectively.
\end{itemize}

Fig. \ref{fig_three_distribution} gives an intuitive illustration of the three different kinds of distributions. In the experiments, we set $\lambda=3$ and $s_{max}=0.1$ based on the parameter evaluation results on MNIST dataset. For the parameter setting of the server, we set $\gamma_{1}=10^{5}$ and $\gamma_{2}=1$. The payment of using one generated data sample which is obtained by leveraging AIGC service is set as $s_{AI}=0.5$. For ease of comparison, datasize for each client $k$ is set to be $d_k \in (100, 300)$ for all the three cases with different distributions of $s_k$ and $\lambda_{k}$.

\begin{figure*}[t]
	\begin{minipage}[t]{0.33\textwidth}
		\setlength{\abovecaptionskip}{-0.05cm}
		\centering
		\includegraphics[width=0.93\linewidth]{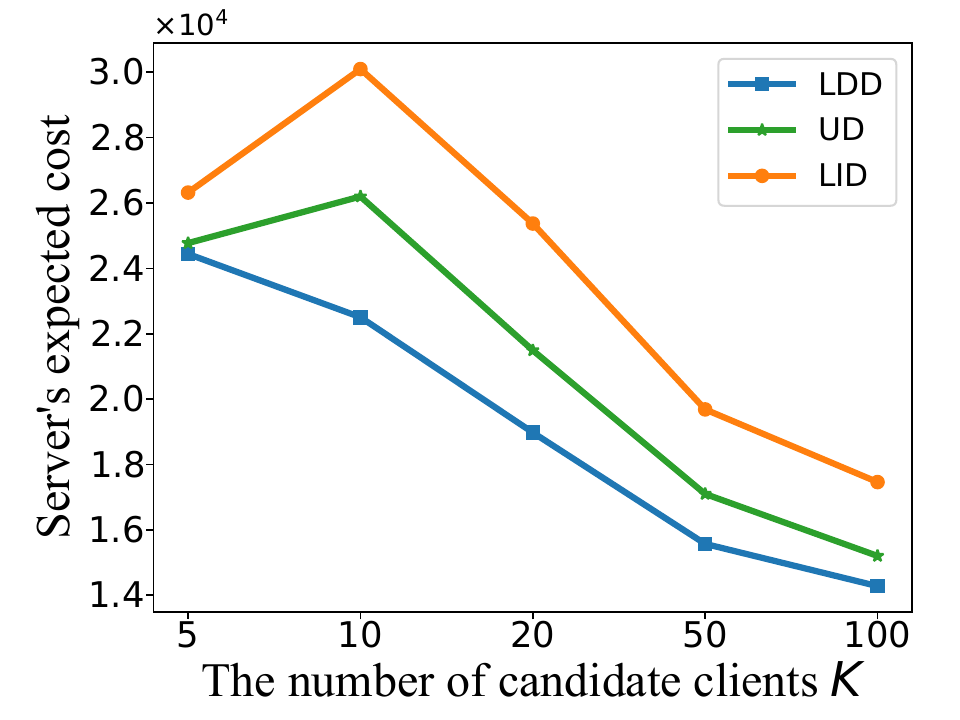}
		\caption{The server's expected cost for each distribution under different number of candidate clients $K$.}
		\label{2_1_cost}
	\end{minipage}
	\begin{minipage}[t]{0.33\textwidth}
		\setlength{\abovecaptionskip}{-0.05cm}
		\centering
		\includegraphics[width=0.93\linewidth]{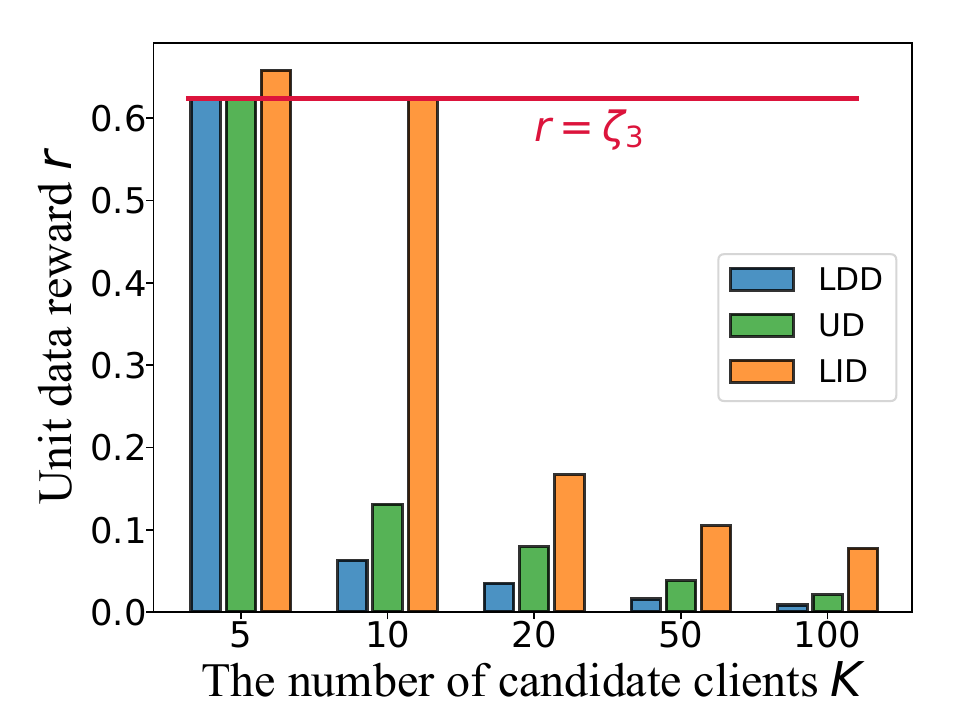}
		\caption{The unit data reward benchmark for each distribution under different number of candidate clients $K$.}
		\label{2_1_r}
	\end{minipage}
	\begin{minipage}[t]{0.33\textwidth}
		\setlength{\abovecaptionskip}{-0.05cm}
		\centering
		\includegraphics[width=0.93\linewidth]{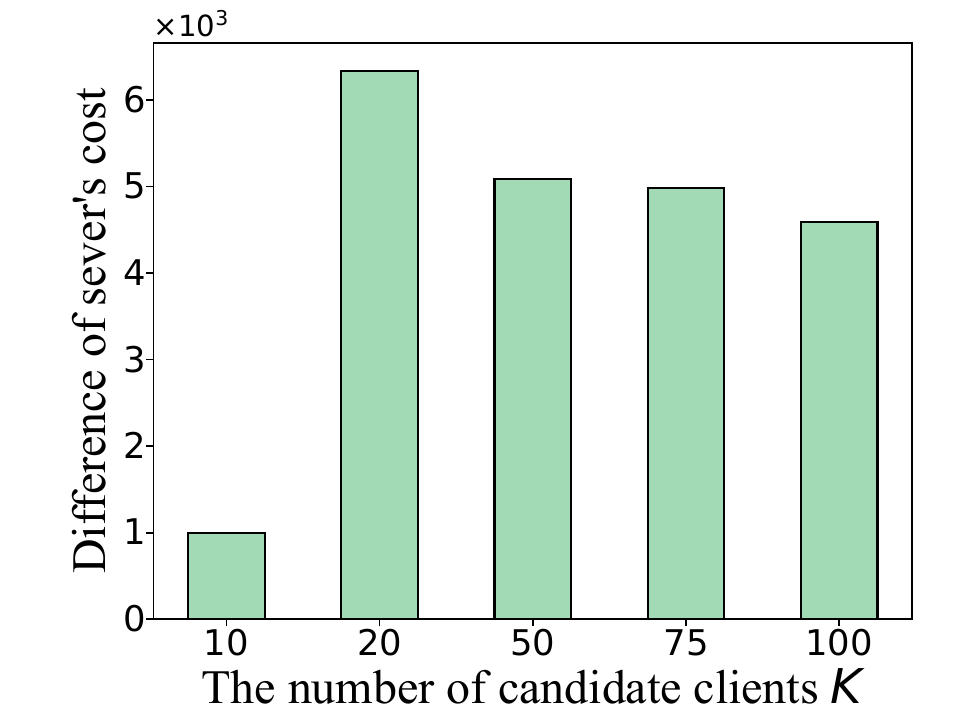}
		\caption{The difference of server's cost between complete and incomplete information scenarios.}
		\label{31_server_cost}
	\end{minipage}
	\vspace{-0.3cm}
\end{figure*}

%To reduce the error caused by randomness, we compute the average of the server's cost and unit data reward benchmark by repeating the calculation process under each distribution case $10$ times.
We conduct experiments under three kinds of distributions of $s_k$ and $\lambda_k$ with different numbers of $K$, and show the changes of server's cost and unit data reward benchmark $r$. To reduce the error caused by randomness, we repeat the calculation process under each distribution case $10$ times and compute the average of the results. As depicted in Fig. \ref{2_1_cost} and Fig. \ref{2_1_r}, our incentive mechanism achieves the lowest server's expected cost in LDD, since the server in LDD scenario can easily recruit high-quality clients with  lower payment cost and hence tends to publish a lower $r$ than UD and LID scenarios. In contrast, the server has to publish the highest uniform unit data reward benchmark $r$ in LID scenario for fear of the case that there are no clients recruited in FL due to a low unit data reward benchmark. With the increase of the number of candidate clients $K$, the number of high-quality clients in the candidate client set increase accordingly. Due to the lower risk of not having enough high-quality clients available for recruitment, the server tends to publish a lower $r$. Besides, the difference of the server's expected costs among the three distributions decrease with the increase of the number of candidate clients when $K \geq 10$.

Interestingly, as shown in Fig. \ref{2_1_r}, the sever adopts $r\geq \zeta_{3}$ as the unit data reward benchmark when $K = 5$. The reason is twofold: 1) with a small number of candidate clients, the server faces the risk of not recruiting enough clients for FL model training if an adequately high reward is not offered; 2) since the number of candidate clients is small, the server is inclined to incentivize clients to participate in FL with their AIGC-enhanced datasets. This implies that introducing generated data in FL possesses the capacity to significantly reduce server expenses, especially in scenarios characterized by limited candidate clients with lower data quality.

%AIGCxuanze

\begin{table}[t]
	\setlength{\abovecaptionskip}{-0.02cm}
	\renewcommand{\arraystretch}{1.3}
	\caption{The average unit data reward benchmark $r$ for different information scenarios.}
	\label{table_r}
	\centering
	\scriptsize
	\begin{tabular}{c|c|c|c|c|c}
		\hline
		& $K=10$ & $K=20$  & $K=50$ & $K=75$&$K=100$\\
		\hline
		Complete  &$0.3910$& $0.2705$& $0.0209$& $  0.0154$&$0.0272$ \\
		\hline
		Incomplete  &$\boldsymbol{0.6238}$& $0.0812$&$0.0395$ &$0.0280$ &$0.0218$\\
		\hline
	\end{tabular}
\end{table}
\subsection{Difference of Server's Cost between Complete and Incomplete Information Scenarios}
To study the difference of server's cost under complete and incomplete information scenarios, the datasize of each client is set to be $d_k = 30$ to remove the impact of client's datasize. For each client $k$, we set $s_{k} \sim Uniform(0,0.1)$ and $\lambda_{k}\sim Uniform(0, 3)$ based on the parameter estimation results of MNIST dataset. For the parameter setting of the server, we set $\gamma_{1}=0.8\cdot 10^{5}$ and $\gamma_{2}=1$. The payment of using one data sample generated by AIGC service is set as  $s_{AI}=0.8$.

Fig. \ref{31_server_cost} and Table \ref{table_r} show the difference of server's cost between complete and incomplete information scenario and the corresponding averaged unit data reward benchmark $r$ for the two cases. To reduce the error caused by randomness, we conduct experiments $10$ times and average the results for the two cases. As depicted in Fig. \ref{31_server_cost}, when $K \geq 20$, the difference of server's costs between complete and incomplete information scenarios decreases with the increase of the number of candidate clients $K$. When $K = 10$, the difference of server's costs between the two information scenarios is the lowest, since the server recruits clients with AIGC-enhanced datasets by setting $r=\zeta_{3}=0.6238$ in incomplete information scenario. In essence, in instances where the amount of candidate clients is limited, the introduction of AIGC-enhanced datasets serves to alleviate server's costs stemming from information asymmetry. However, as the quantity of candidate clients escalates, there's a rise in the presence of high-quality clients (exhibiting lower $\lambda_{k}$ and lower $s_{k}$). Consequently, the server's reliance on generated data diminishes. Conversely, offering a uniformly increased $r$ to a larger number of clients which use AIGC-enhanced datasets will result in heightened payments to clients.

\begin{figure*}[t]
	\begin{minipage}[t]{0.33\textwidth}
		\setlength{\abovecaptionskip}{-0.05cm}
		\centering
		\includegraphics[width=\linewidth]{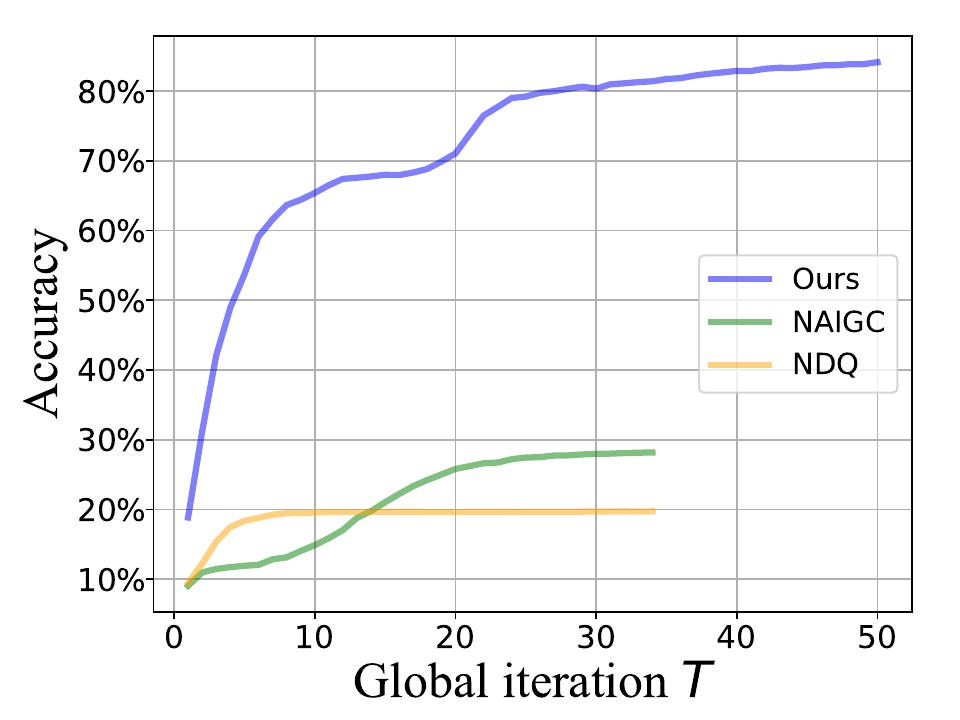}
		\caption{The accuracy of each mechanism on MNIST dataset.}
		\label{Mnist_acc}
	\end{minipage}
	\begin{minipage}[t]{0.33\textwidth}
		\setlength{\abovecaptionskip}{-0.05cm}
		\centering
		\includegraphics[width=\linewidth]{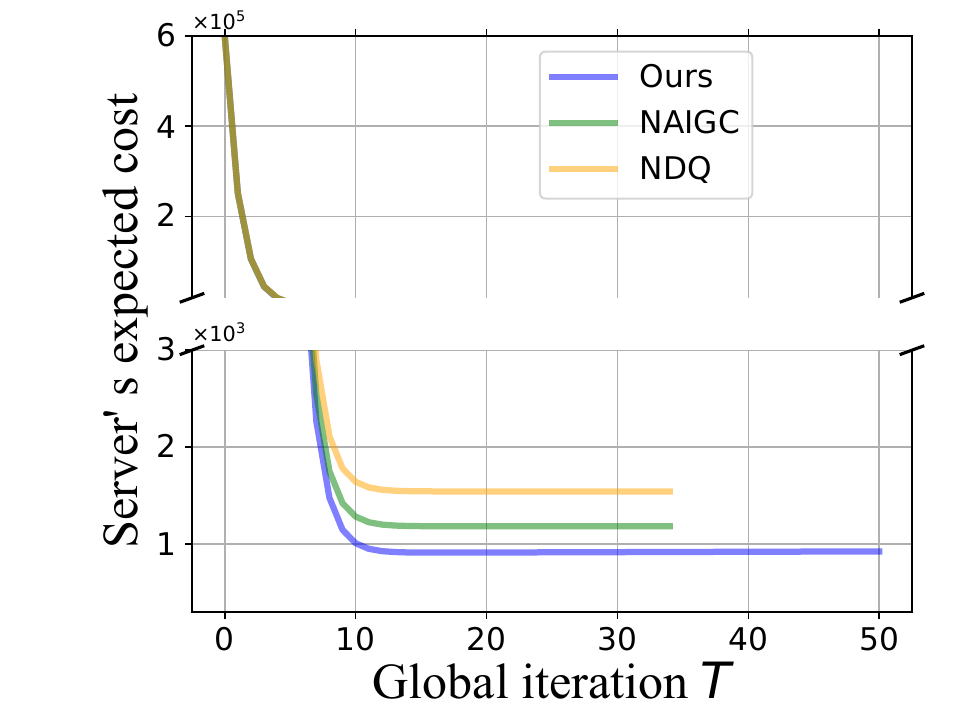}
		\caption{The server's expected cost of each mechanism on MNIST dataset.}
		\label{Mnist_cost}
	\end{minipage}
	\begin{minipage}[t]{0.33\textwidth}
		\setlength{\abovecaptionskip}{-0.05cm}
		\centering
		\includegraphics[width=\linewidth]{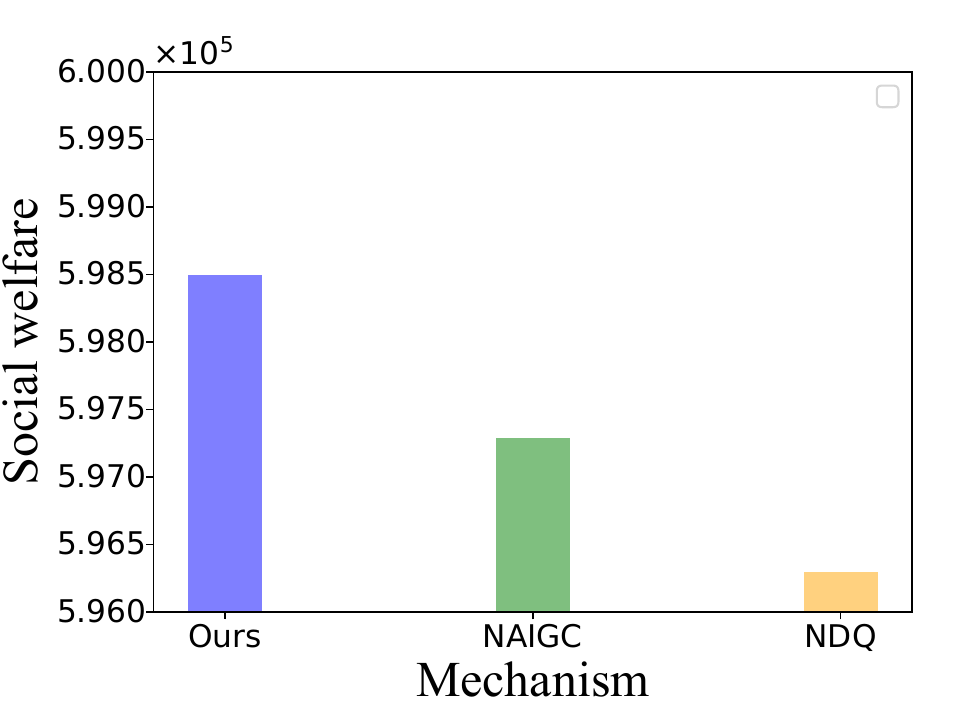}
		\caption{The social welfare of each mechanism on MNIST dataset.}
		\label{minist_social}
	\end{minipage}
	\vspace{-0.3cm}
\end{figure*}
\begin{figure*}[t]
	\begin{minipage}[t]{0.33\textwidth}
		\setlength{\abovecaptionskip}{-0.05cm}
		\centering
		\includegraphics[width=\linewidth]{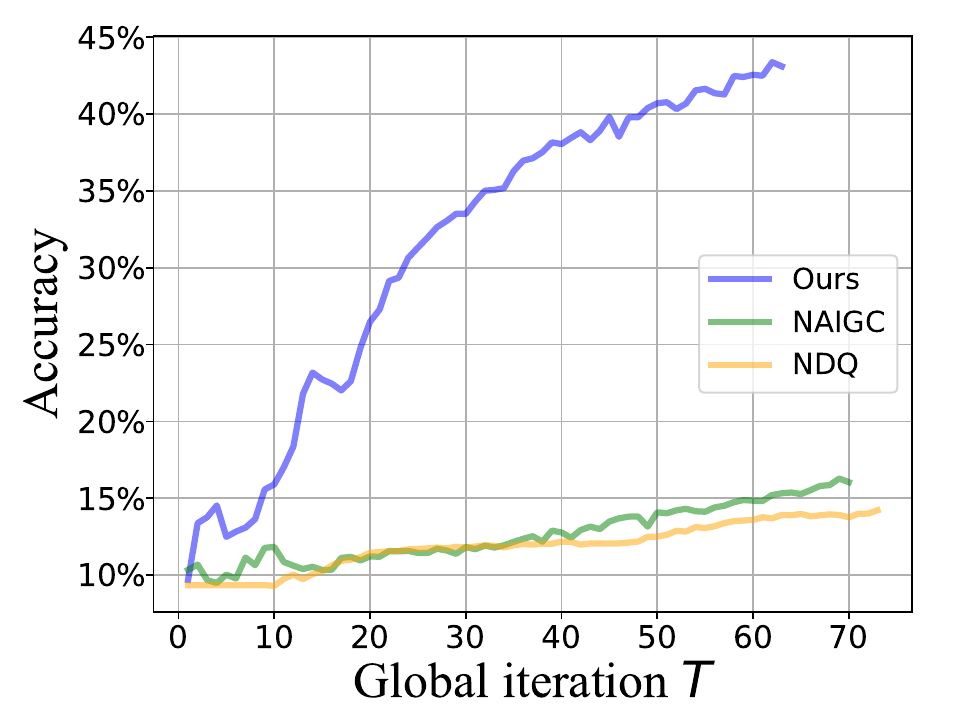}
		\caption{The accuracy of each mechanism on CIFAR10 dataset.}
		\label{cifar_acc}
	\end{minipage}
	\begin{minipage}[t]{0.33\textwidth}
		\setlength{\abovecaptionskip}{-0.05cm}
		\centering
		\includegraphics[width=\linewidth]{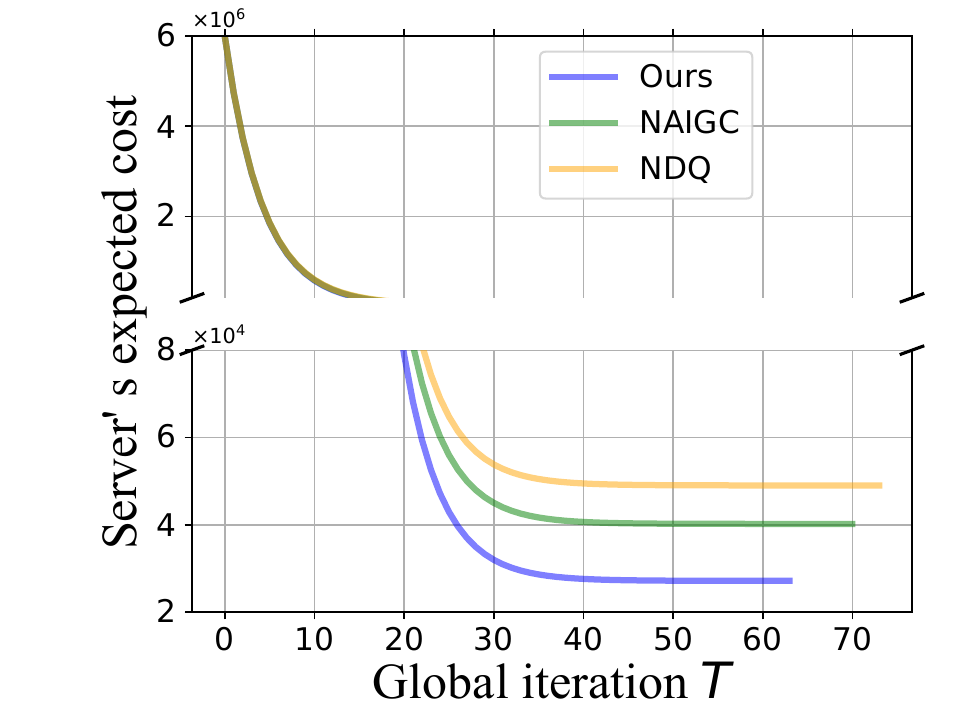}
		\caption{The server's expected cost of each mechanism on CIFAR10 dataset.}
		\label{cifar_cost}
	\end{minipage}
	\begin{minipage}[t]{0.33\textwidth}
		\setlength{\abovecaptionskip}{-0.05cm}
		\centering
		\includegraphics[width=\linewidth]{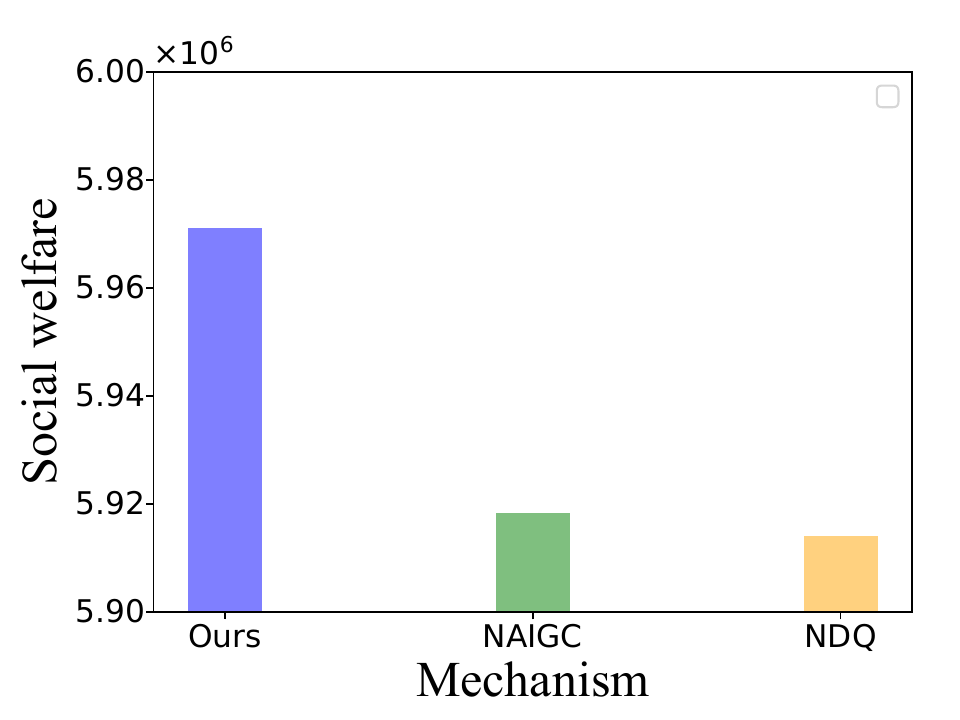}
		\caption{The social welfare of each mechanism on CIFAR10 dataset.}
		\label{cifar_social}
	\end{minipage}
	\vspace{-0.3cm}
\end{figure*}
\begin{figure*}[t]
	\begin{minipage}[t]{0.33\textwidth}
		\setlength{\abovecaptionskip}{-0.05cm}
		\centering
		\includegraphics[width=\linewidth]{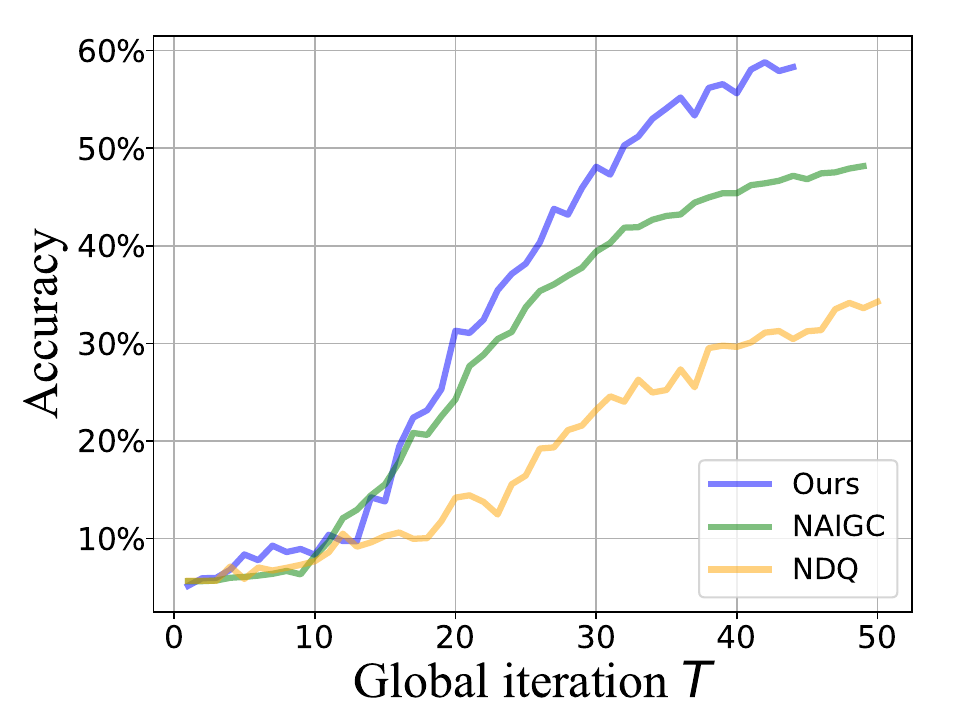}
		\caption{The accuracy of each mechanism on GTSRB dataset.}
		\label{gtsrb_acc}
	\end{minipage}
	\begin{minipage}[t]{0.33\textwidth}
		\setlength{\abovecaptionskip}{-0.05cm}
		\centering
		\includegraphics[width=\linewidth]{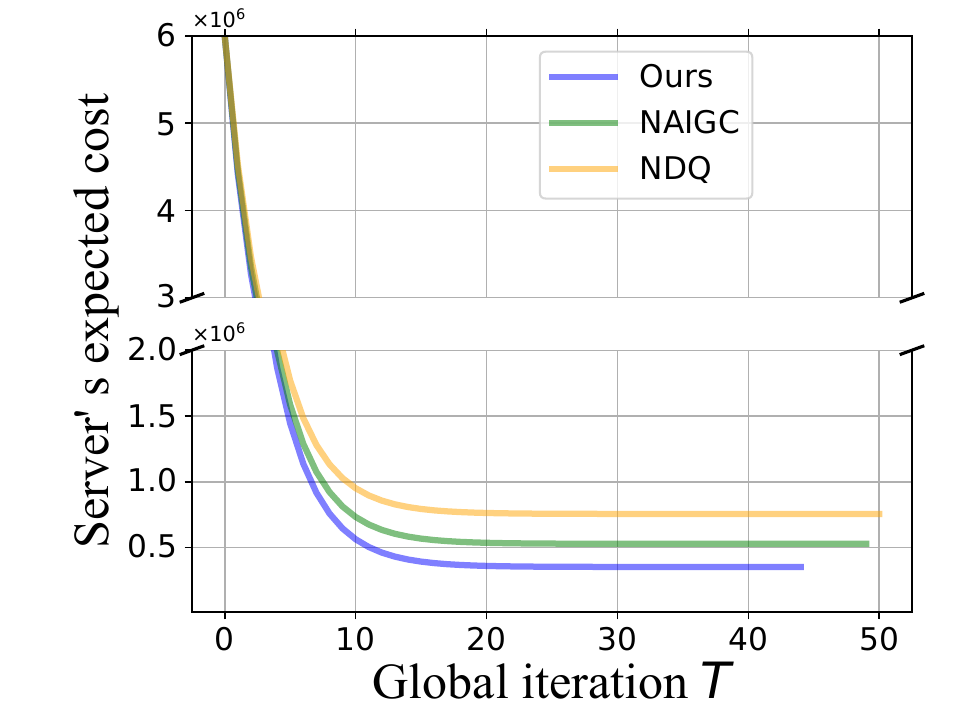}
		\caption{The server's expected cost of each mechanism on GTSRB dataset.}
		\label{gtsrb_cost}
	\end{minipage}
	\begin{minipage}[t]{0.33\textwidth}
		\setlength{\abovecaptionskip}{-0.05cm}
		\centering
		\includegraphics[width=\linewidth]{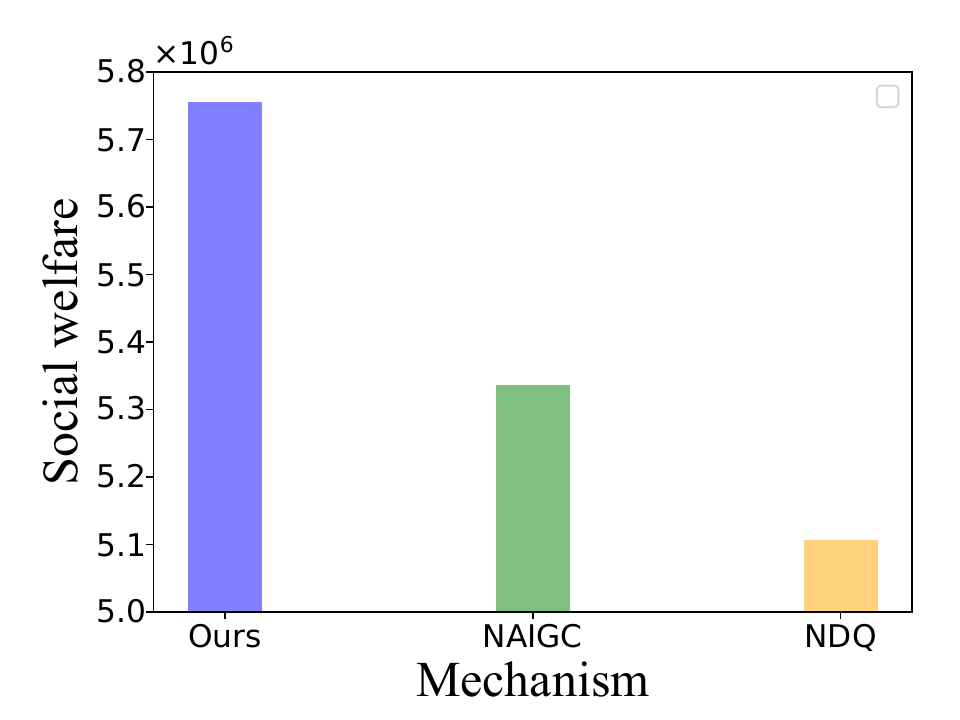}
		\caption{The social welfare of each mechanism on GTSRB dataset.}
		\label{gtsrb_social}
	\end{minipage}
	\vspace{-0.3cm}
\end{figure*}

\subsection{Training Performance Comparison}
We compare the training performance of our proposed mechanism with the following two benchmarks:

\begin{itemize}
	\item \textit{No AIGC (NAIGC)}: The server considers data quality of clients while neglecting that clients may choose to obtain a high-quality dataset by using AIGC service.
	%obtain a high-quality generated data by using AIGC service.
	\item \textit{No data quality (NDQ)}: The server neglects the data quality of the clients and assigns payment $rd_{k}$ to each client $k$.	
\end{itemize}

We consider FL scenarios with $K=30$ clients on MNIST, CIFAR10 and GTSRB datasets.
For the experiments on MNIST dataset, the server's parameters are set as $\gamma_{1}=10^{5}$ and $\gamma_{2}=0.01$. We set $\lambda_{k}\sim uniform(0,3)$ and $s_{k}\sim uniform(0,10^{-3})$ for each client $k$. Based on the value of $\lambda_{k}$, we can find the number of classes $l$ for the data samples of client $k$ from Table \ref{table_para1}, and hence assign $d_k \in  [100,300]$ data samples from $l$ classes randomly for each client $k$. As for CIFAR10 dataset, we set $\gamma_{1}=10^{6}$ and $\gamma_{2}=0.01$ for the server, and set $\lambda_{k}\sim uniform(0,2)$ and $s_{k}\sim uniform(0,5\cdot 10^{-3})$ for each client $k$. According to Table \ref{table_para2}, we can find the number of classes $l$ for the data samples of client $k$, and then assign $d_k \in  [300,500]$ data samples from $l$ classes randomly for each client $k$ based on the value of $\lambda_{k}$. As for GTSRB dataset, we set $\gamma_{1}=10^{6}$ and $\gamma_{2}=0.01$ for the server, and set $\lambda_{k}\sim uniform(0, 10)$ and $s_{k}\sim uniform(0,5\cdot 10^{-3})$ for each client $k$. In light of data samples with large number of classes (e.g., $43$ classes) on GTSRB dataset, to simulate the uniform distribution of $\lambda_{k}$, we assign $d_k \in [1000,2000]$ data samples with a random number of classes $l\in [4, 43]$ for each client $k$. The unit cost or payment for one generated data sample is set to $s_{AI}=0.01$ for MNIST, CIFAR10 and GTSRB dataset.

To show the performance of our proposed mechanism, we evaluate the server's expected cost, training accuracy and social welfare on MNIST, CIFAR10 and GTSRB dataset, respectively. Here, the social welfare is defined as the sum of server's cost reduction and the total utilities of all clients. As shown in Fig. \ref{Mnist_acc} and Fig. \ref{Mnist_cost}, our mechanism achieves the highest accuracy (i.e., $84.15\%$) and the lowest server's cost compared with NAIGC and NDQ methods on MNIST dataset. This is because that NAIGC only considers the data quality of the local data for the clients, without introducing the generated data for data quality improvement, thus resulting in a degraded model accuracy. Also, as NDQ selects clients based on data quantity while neglecting the heterogeneity of data quality among clients, NDQ is inclined to select clients with lower data quality and hence achieves the lowest training accuracy. By introducing data generation capability and data quality awareness in our mechanism, clients can utilize data generation technique to enhance data quality, enabling FL model to converge to superior performance compared to NAIGC and NDQ, while also achieving the highest social welfare.

Similarly, our mechanism surpasses NDQ and NAIGC on the CIFAR10 dataset, achieving the highest training accuracy and social welfare while reducing server costs by 44.61\% and 32.45\%, respectively, as shown in Fig. \ref{cifar_acc}, Fig. \ref{cifar_cost}, and Fig. \ref{cifar_social}. In presence of dataset with more classes, such as GTSRB dataset, our mechanism still exhibits superior performance in terms of training accuracy and social welfare and expected server's cost, as illustrated in Fig.
\ref{gtsrb_acc}, Fig. \ref{gtsrb_cost} and Fig.\ref{gtsrb_social}. Notably, our proposed mechanism achieves cost reductions of 53.34\% and 33.29\% compared to NDQ and NAIGC, respectively. These results underscore the robustness of our proposed mechanism across diverse datasets.

\section{Related Work}
A plethora of studies on federated learning concentrate on improving training efficiency and final model performance \cite{kairouz2021advances, lee2020accurate}. For instance, Jeong et al. devise federated augmentations approach to rectify non-IID dataset for performance improvement \cite{jeong2018communication}. 
By leveraging AIGC service, clients in FL can conduct data synthesis to mitigate data heterogeneity issue \cite{li2023filling}. However, most of the results are derived under an optimistic assumption that clients participate in FL voluntarily, and adopt AIGC services to generate data unconditionally, which may be unrealistic without proper incentives. 

Considering the data-computing cost and data generation cost (incurred by adopting AIGC service) of clients, incentive mechanism design is necessary for the server to compensate the cost of clients reasonably. However, existing works on incentive mechanism design for FL exhibit the following limitations. First, most of them considers only one or two dimensions of clients' information (e.g., data-computing cost, data amount, etc.) for contribution evaluation \cite{le2021incentive, zhang2022robust, wang2023incentive}, which may not be directly applicable in realistic FL scenarios with data and resource heterogeneity. Second, few studies consider to capture the relationship between the quality of local model updates (which is related to local data distribution) and the global learning performance \cite{pandey2020crowdsourcing, zhan2020learning,zeng2020fmore}. Third, most existing incentive mechanisms assume that the server is aware of all clients' attributes, which is unrealistic in practice \cite{Crow2019}, \cite{Auction21}. Fourth, faced the temptation of being rewarded, rational clients may optionally generate a high-quality data to participate in FL with a higher reward \cite{li2023filling}, which further complicates the incentive mechanism design due to the complex client behaviors and is less understood in existing studies.

There exist some emerging works attempt to deal with the aforementioned issues in different manners. For example, the authors study incentive mechanism with multi-dimensional clients' private information under different levels of information
asymmetry by contract theory \cite{ding2020optimal, ding2020incentive, zeng2020fmore}. The author propose a multi-dimensional procurement auction for incentive mechanism in FL based on auction analysis framework. Along a different line, in this paper we propose a lightweight data quality-aware incentive mechanism for AIGC-empowered FL.

\section{Conclusion}
In this paper, we propose a data quality-aware incentive mechanism to encourage clients' participation in AIGC-empowered FL scenario. With rigorous analysis of convergence performance of FL model trained using a blend of real-world and generated data samples, we derive the optimal server's incentive strategies both in complete and incomplete information scenarios. Extensive experimental results demonstrate that introducing AIGC service for FL scenarios enables significant cost reduction for the server.

\bibliographystyle{unsrt}
\bibliography{reference} 

\vfill

\end{document}